\documentclass[11pt,letterpaper]{article}

\PassOptionsToPackage{numbers,sort&compress}{natbib}

\usepackage[margin=1in]{geometry}
\usepackage{natbib}

\usepackage[utf8]{inputenc}
\usepackage[T1]{fontenc}

\usepackage[dvipsnames]{xcolor}
\definecolor{Bleu}{RGB}{0,0,204}

\usepackage{hyperref}
\hypersetup{
  colorlinks=true,
  citecolor=Bleu,
  linkcolor=Bleu,
  urlcolor=Bleu,
  hypertexnames=false,
  pdftitle={Inverse Reinforcement Learning with Just Classification and a Few Regressions},
  pdfauthor={Lars van der Laan; Nathan Kallus; Aurelien Bibaut}
}

\usepackage{amsmath,amssymb,amsthm,mathtools}
\usepackage{amsfonts}
\usepackage{bbm}
\usepackage{algorithm}
\usepackage{algorithmic}
\usepackage{booktabs}
\usepackage{wrapfig}
\usepackage{graphicx}
\usepackage{enumitem}
\usepackage{microtype}
\usepackage{multirow}
\usepackage{url}
\usepackage{cleveref}
\usepackage{nicefrac}

\newtheorem{theorem}{Theorem}
\newtheorem{lemma}{Lemma}
\newtheorem{assumption}{Assumption}

\newcommand{\E}{\mathbb{E}}

\newcommand{\Prob}{\mathbb{P}}

\newcommand{\KL}{\mathrm{KL}}

\newcommand{\LSE}{\operatorname{LSE}}

\newcommand{\ustar}{u^\star}

\title{Inverse Reinforcement Learning with Just Classification and a Few Regressions
}

\author{%
Lars van der Laan$^{1}$\qquad Nathan Kallus$^{2,3}$\qquad Aurelien Bibaut$^{2}$\\[0.5em]
\normalsize $^{1}$Department of Statistics, University of Washington\\
\normalsize $^{2}$Netflix\qquad $^{3}$Cornell University\\[0.3em]
\small\texttt{lvdlaan@uw.edu}\quad\texttt{nkallus@cornell.edu}\quad\texttt{abibaut@netflix.com}
}
\date{}

\begin{document}

\maketitle
\begin{abstract}
Inverse reinforcement learning (IRL) seeks to recover a reward function from
observed behavior, but rewards are typically only partially identified: many
reward--value pairs can induce the same behavior policy. We study this problem
in the maximum-entropy, or Gumbel-shock, model under a broad class of statewise
affine normalization constraints, with anchor-action constraints as a special
case. This leads to Generalized Policy-to-$Q$-to-Reward (GenPQR), a
modular approach to normalized reward recovery based on policy estimation and
$Q$-evaluation via the Bellman equation; both components can be instantiated
with off-the-shelf classification and regression methods. We establish modular
finite-sample guarantees under general function approximation, with separate
terms for policy-estimation and $Q$-function estimation error. We also give explicit bounds for GenPQR with fitted $Q$-evaluation. In experiments, GenPQR improves on
DeepPQR in reward recovery. Our guarantees extend beyond anchor actions, accommodate
large action spaces, and are not tied to a specific neural-network architecture
or training procedure.
\end{abstract}

\section{Introduction}
\label{sec:motivation}

Behavioral data are abundant in robotics, economics, healthcare, and
human--computer interaction. Inverse reinforcement learning (IRL) seeks to
explain such behavior by recovering a reward under which the observed policy is
optimal. Classical IRL often treats agents as exactly optimal
\citep{ng2000algorithms,abbeel2004apprenticeship}, but this deterministic view
can miss the variability in real behavior. A common alternative is stochastic
choice, for example through entropy regularization, which yields softmax-type
policies
\citep{ziebart2008maximum,ziebart2010modeling,haarnoja2017reinforcement}. In
maximum-entropy (MaxEnt) IRL \citep{ziebart2008maximum}, closely related to
dynamic discrete-choice (DDC) models with i.i.d.\ Gumbel shocks
\citep{rust1987optimal}, the observed policy has a softmax form induced by an
unknown reward and continuation value.

Even in the softmax setting, rewards are only partially identified: different
reward--value pairs can induce the same behavior policy through potential-based
shaping transformations
\citep{ng1999policy,cao2021identifiability,skalse2023invariance,skalse2024partial}.
Thus policy fit alone does not imply reward recovery. Existing methods obtain
unique rewards only by adding restrictions, often implicitly: MaxEnt IRL
identifies an equivalence class unless the reward class is restricted, for
example to be linear \citep{ziebart2008maximum}; adversarial IRL requires
strong conditions for reward recovery, such as state-only rewards and
deterministic transitions \citep{fu2018airl}; and other neural IRL objectives
depend on their chosen reward parameterization
\citep{levine2011nonlinear,wulfmeier2015deep,ho2016gail,
snoswell2020revisiting}. This motivates normalized reward recovery, a classical idea in DDC/econometrics:
impose an explicit identifying normalization and recover the unique
reward--value pair satisfying it
\citep{rust1987optimal,hotz1993conditional,aguirregabiria2010dynamic,
geng2020deep}.

Building on classical DDC for discrete actions, Deep Policy-to-\(Q\)-to-Reward
(DeepPQR) \citep{geng2020deep} implements this idea under anchor-action
normalization. Given an observed behavior policy, or a policy estimated by an
upstream method such as imitation learning or adversarial IRL, DeepPQR estimates
the associated \(Q\)-function using value-based offline RL tools and then
recovers rewards from the normalized Bellman equation. The procedure
uses only the upstream policy, not the upstream reward model, so reward recovery
does not rely on the correctness of that model. Its identification strategy,
however, is tied to a fixed anchor action: one must specify a well-supported
reference action, typically a do-nothing action, whose reward is known. This can
be restrictive when actions are continuous, weakly supported, or lack a canonical
reference option. Other settings may instead motivate normalizations based on
mean rewards, state-dependent anchors, or value-based constraints.

We extend this policy-to-\(Q\)-to-reward perspective to statewise affine
normalizations, which identify a unique reward without requiring a fixed
anchor action. Given an estimated behavior policy, GenPQR solves the
corresponding Bellman fixed point for \(Q\) and recovers the normalized reward
from the Bellman equation. This retains the modular structure of DeepPQR while
allowing normalization choices beyond a single reference action.

\textbf{Our contributions.}
First, we characterize reward identification in maximum-entropy IRL,
equivalently the Gumbel-shock discrete-choice model. We show that behavior
identifies only an equivalence class of reward--value pairs, and introduce
statewise affine normalizations that select a unique representative. This
generalizes the anchor-action normalization of DeepPQR \citep{geng2020deep} and
clarifies when exact reward recovery is possible, and when it is unnecessary for
policy comparison.

Second, this characterization yields a modular identification and estimation
strategy. Under any statewise affine normalization, reward recovery reduces to
estimating the behavior policy, solving a linear Bellman fixed point for an
associated \(Q\)-function, and applying the normalized Bellman equation. This
gives \emph{Generalized Policy-to-\(Q\)-to-Reward} (GenPQR), which treats reward
recovery as \(Q\)-evaluation after policy estimation rather than as a specialized
joint IRL objective. In the fixed-anchor setting, GenPQR recovers rewards
directly from \(Q\), avoiding the additional reward-regression step in DeepPQR;
DeepPQR can be viewed as one neural-network implementation of the
\(Q\)-estimation step. More generally, the same reduction covers
state-dependent anchors, mean-reward and value normalizations, and large or
continuous action spaces through the choice of policy and \(Q\)-estimation
methods.

Third, we prove finite-sample guarantees for GenPQR under function
approximation. The bounds apply to generic policy and \(Q\)-function estimators;
for FQE, they separate policy-estimation, Bellman-approximation, statistical,
and iteration errors. Relative to the DeepPQR analysis, our theory replaces
sup-norm policy-error conditions with weaker error measures, makes coverage
explicit, allows approximate rather than exact Bellman completeness, and is not
tied to a particular neural-network architecture or training procedure.

Fourth, our experiments show that GenPQR remains accurate as the number of
actions grows. In a 20-action data-fusion setting, it combines
state--action transitions with an
independent state--outcome sample to improve reward recovery beyond
anchor-action methods using only the behavior sample.

\subsection{Related Work}

\textbf{Identifiability, shaping, and anchor-action methods.}
Rewards in MaxEnt IRL are only partially identified because behavior is
invariant under potential-based shaping
\citep{ng1999policy,fu2018airl,cao2021identifiability,skalse2023invariance,
skalse2024partial}. Our work is closest to DeepPQR \citep{geng2020deep}, which
studies anchor-action normalization and gives finite-sample guarantees for a
specific neural-network procedure. We extend the anchor-action view to general
affine normalizations, clarify the identification structure, simplify the
recovery step (Section~\ref{sec:solve-normalization}), and develop theory under
general function approximation.

\textbf{Behavioral and statistical misspecification.}
Behavioral misspecification arises when the observed policy is not generated by
the assumed decision model, whereas statistical misspecification arises when
the policy or \(Q\)-function classes cannot represent their population targets.
Any full-support policy admits a MaxEnt rationalization, and our normalization
selects one reward from the resulting shaping class. Under behavioral
misspecification, this model-implied reward remains well defined but need not
equal the agent's structural reward
\citep{skalse2023misspecification,skalse2024partial}. Our finite-sample bounds
separately quantify policy-estimation and Bellman-approximation errors within
the MaxEnt/Gumbel-shock model.

\textbf{MaxEnt IRL and adversarial imitation learning.}
Maximum-entropy IRL fits stochastic policies induced by soft Bellman equations,
often under structured reward parameterizations
\citep{ziebart2008maximum,ziebart2010modeling,levine2011nonlinear,
wulfmeier2015deep,zeng2022maximum}. Adversarial methods such as GAIL and AIRL
are effective for imitation learning and behavior-policy estimation through
joint reward-policy optimization \citep{ho2016gail,fu2018airl}. As noted by
\citet{geng2020deep}, however, they generally do not resolve reward
nonidentifiability without stronger assumptions, such as state-only rewards. We
instead separate behavior-policy estimation from normalized reward recovery,
allowing action-dependent rewards. Thus, methods designed primarily to reproduce
behavior can still serve as the first stage of a modular reward-recovery
procedure.

\textbf{Entropy-regularized RL and control as inference.}
Our analysis is also connected to entropy-regularized control and the
control-as-inference perspective
\citep{kappen2005linear,todorov2009efficient,levine2018rlasinf}, including
path-consistency learning \citep{nachum2017pcl}, soft actor-critic
\citep{haarnoja2018sac}, and entropy-regularized offline RL
\citep{haarnoja2017reinforcement,uehara2023offline}. We use the same soft
Bellman structure, but for the inverse problem: the behavior policy is assumed
to solve an entropy-regularized control problem for an unknown reward, and the
goal is to recover that reward.

\textbf{Value-based offline RL.}
The recovery step is closely connected to value-based offline RL. Fitted
\(Q\)-iteration and fitted \(Q\)-evaluation estimate Bellman fixed points by
regression \citep{ernst2005tree,munos2008finite,mnih2013playing,van2025fitted},
while minimax and critic-based methods relax completeness assumptions
\citep{uehara2020minimax,imaizumi2021minimax,uehara2023offline,xie2020q,
xie2021batch,zhan2022offline}. We show that normalized reward recovery
reduces to solving a linear fixed-point equation with these methods. Our finite-sample analysis separates
first-stage policy-estimation error from second-stage value-learning error.

\section{Problem Setup}
\label{sec:background}

We consider a discounted MDP with state space \(\mathcal S\), finite or
continuous action space \(\mathcal A\), transition kernel \(P\), reward
\(r^\dagger:\mathcal S\times\mathcal A\to\mathbb R\), and discount
\(\gamma\in[0,1)\). Let \(\pi(a\mid s)\) be the behavior policy and \(\rho\)
the sampling distribution over states. We observe transitions
\(\{(s_i,a_i,s_i')\}_{i=1}^n\) with
\[
s_i\sim \rho,\qquad
a_i\sim \pi(\cdot\mid s_i),\qquad
s_i'\sim P(\cdot\mid s_i,a_i).
\]
Equivalently, \((s_i,a_i)\sim \nu_\pi\), where
\(\nu_\pi(ds,a):=\rho(ds)\pi(a\mid s)\). Thus \(P\) is identified from the
observed dynamics, whereas the reward is not. We use finite-action notation in
the main text, following \citet{rust1987optimal,geng2020deep}; continuous
actions replace sums and softmax normalizers by integrals and Boltzmann
densities with respect to a reference measure
(Appendix~\ref{app:continuous-actions}). For any state-action function \(f\),
write
\[
(\mu f)(s):=\sum_a \mu(a\mid s) f(s,a).
\]
Throughout the identification analysis, \(\pi(a\mid s)>0\) for every action
and \(\rho\)-almost every state, \(u^\star=\log\pi\) is bounded, and the
one-step coverage condition \(\nu_\pi P\ll\rho\) holds. Statewise conclusions
are understood \(\rho\)-almost everywhere and state--action conclusions
\(\nu_\pi\)-almost everywhere.
 The \textbf{goal} of IRL is to recover a reward \(r\) for which \(\pi\) is
optimal in the MDP \((\mathcal S,\mathcal A,P,r,\gamma)\) under the
entropy-regularized criterion below. We review the MaxEnt IRL setting from structural discrete-choice and
maximum-entropy perspectives, then formulate the optimization problem central to
our analysis.

\textbf{From dynamic discrete choice to MaxEnt IRL.}
We adopt the dynamic discrete-choice formulation
\citep{rust1987optimal,hotz1993conditional,aguirregabiria2010dynamic}. At time
\(t\), an agent in state \(s_t\) who takes action \(a_t\) receives utility
\(r^\dagger(s_t,a_t)+\varepsilon_t(a_t)\), where \(r^\dagger\) is the unknown
mean reward and \(\varepsilon_t(a)\) is an idiosyncratic shock with known
distribution. Let \(V^\dagger(s)\) be the optimal ex ante value, define
\(Pf(s,a):=E\{f(s')\mid s,a\}\), and set
\[
Q^\dagger(s,a):=r^\dagger(s,a)+\gamma PV^\dagger(s,a).
\]
For shock scale \(\tau>0\), define
\[
\LSE_\tau f(s):=\tau\log\sum_a \exp\{f(s,a)/\tau\}.
\]
Under i.i.d.\ Gumbel type-I extreme-value shocks with scale \(\tau\), the
optimal policy is softmax:
\[
\pi^\dagger(a\mid s)\propto \exp\{Q^\dagger(s,a)/\tau\},
\]
and
\[
Q^\dagger = r^\dagger+\gamma P\LSE_\tau Q^\dagger,
\qquad
V^\dagger(s)=\LSE_\tau Q^\dagger(s).
\]
Equivalently, with state-action continuation value
\(v^\dagger:=P\LSE_\tau Q^\dagger\),
\[
\pi^\dagger(a\mid s)
\propto
\exp\{(r^\dagger(s,a)+\gamma v^\dagger(s,a))/\tau\},
\qquad
v^\dagger=P\LSE_\tau(r^\dagger+\gamma v^\dagger).
\]
We call the latter the soft Bellman equation. We set \(\tau=1\), write
\(\LSE:=\LSE_1\), and absorb the known scale into
\(r^\dagger\).

 An equivalent perspective comes from maximum-entropy IRL
\citep{ziebart2008maximum,ziebart2010modeling}, where the agent maximizes
expected discounted reward plus an entropy bonus for stochastic action
selection. This yields the same optimal policy \(\pi^\dagger\) and soft Bellman
equation, with \(V^\dagger\) and \(Q^\dagger\) the entropy-regularized value
and \(Q\)-functions \citep{haarnoja2017reinforcement}. Thus, despite different
motivations, dynamic discrete choice and MaxEnt IRL reduce to the same
mathematical object: a soft Bellman system with a softmax policy.

\textbf{Partial identification and the role of normalization.}
Our goal is therefore to recover a reward function whose induced soft-optimal
policy best matches the observed behavior, for example by minimizing the
state-averaged Kullback--Leibler divergence from \(\pi(\cdot\mid s)\) to
\(\pi^\star(\cdot\mid s)\). However, matching the policy does not in general
identify the reward. Under a softmax policy, adding a state-dependent offset to
all action values leaves the policy unchanged, so many rewards induce the same
behavior policy
\citep{ziebart2010modeling,cao2021identifiability}. The reward is therefore
only \emph{partially identified}, up to an equivalence class, and selecting a
unique representative requires a normalization constraint \citep{rust1987optimal,geng2020deep}.

We use \textbf{statewise affine normalizations}, which generalize standard
choices such as anchor-action, outside-option, and sum-to-zero constraints. Let
\(\mu(\cdot\mid s)\) be a reference distribution over actions and
\(g:\mathcal S\to\mathbb R\) a specified anchor function. We impose
\begin{equation}
\sum_a \mu(a\mid s) r(s,a)=g(s)
\qquad \text{for } \rho\text{-almost every }s,
\end{equation}
or equivalently \(\mu r=g\), with \(\sup_s |g(s)|<\infty\). Thus, the
\(\mu\)-average reward at each state is fixed at \(g(s)\), with \(\mu\) and
\(g\) typically chosen from domain knowledge.

Such constraints are standard in economics and are often substantively
meaningful \citep{hotz1993conditional,bajari2010identification}. The fixed
anchor-action constraint \(r(s,a^\dagger)=g(s)\) is the special case
\(\mu(a\mid s)=1\{a=a^\dagger\}\) \citep{geng2020deep}; when \(g\equiv0\), it
reduces to the classical zero-reward normalization \(r(s,a^\dagger)=0\) used in
Rust's engine-replacement model \citep{rust1987optimal}. More generally, taking
\(\mu(a\mid s)=1\{a=a^\dagger(s)\}\) allows state-specific anchors, where the
reference action may vary with the state. DDC models often use such anchors to
normalize the payoff of a no-action, status-quo, or outside option to be zero
or otherwise known
\citep{rust1987optimal,hotz1993conditional,aguirregabiria2010dynamic,
geng2020deep}. If \(\mu(a\mid s)=1/|\mathcal A|\), the constraint fixes the
statewise average reward and, when \(g\equiv0\), becomes a sum-to-zero
constraint \citep{kallus2016revealed}.

The framework also permits data-driven choices of \((\mu,g)\). A statewise
value constraint is a special case: if \(V_r^\mu\) is the value function of
policy \(\mu\) under reward \(r\), then imposing \(V_r^\mu=h\) is equivalent,
by the Bellman equation \(V_r^\mu=\mu r+\gamma\mu P V_r^\mu\), to taking
\(g:=h-\gamma\mu P h\). For example, in medical decision problems, one may wish
to normalize rewards relative to standard care using historical outcome data
\citep{kallus2018removing}. If an auxiliary dataset records states and realized
outcomes \((s_i,y_i)\) from a population with behavior policy \(\pi\), where
\(E[y\mid s,a]=r(s,a)\), then
\(g(s)=E[y\mid s]=\sum_a \pi(a\mid s)r(s,a)\). Taking \(\mu=\pi\) therefore
yields a natural data-driven normalization.

\textbf{Main problem:} Putting these pieces together, we obtain the following
constrained maximum-likelihood problem for recovering \((r^\star, v^\star)\):
\begin{equation}
\label{eq:main-irl}
\begin{aligned}
\arg\max_{r,v}\quad
& \E_{(s,a) \sim \nu_{\pi}}\!\left[
  r(s,a) + \gamma v(s,a) - \LSE(r+\gamma v)(s)
\right] \\
\text{s.t.}\quad
& v = P\LSE(r+\gamma v), \qquad\qquad\text{\color{gray}(soft Bellman)} \\
& \mu r = g.
\qquad\text{\color{gray}(affine normalization)}
\end{aligned}
\end{equation}
That is, we maximize the conditional log-likelihood over state--action
functions \(r\) and \(v\) subject to the soft Bellman equation and a statewise
affine normalization constraint.

The remainder of the paper shows that solving \eqref{eq:main-irl} reduces to two steps: estimate \(\pi\), then solve a linear Bellman evaluation equation to recover the unique pair \((r^\star, v^\star)\). We then turn this characterization into a simple algorithm.

\section{From partial to point identification}
\label{sec:char}

To expose the core structure of the problem, we first drop the normalization constraint and study the resulting relaxed optimization problem.

\textbf{Relaxed problem:} We remove the reward normalization in
\eqref{eq:main-irl} but keep soft Bellman consistency:
\begin{equation}
\label{eq:relaxed-irl}
\begin{aligned}
\arg\max_{r,v}\quad
& \E_{(s,a) \sim \nu_{\pi}}\!\left[
  r(s,a) + \gamma v(s,a) - \LSE(r+\gamma v)(s)
\right] \\
\text{s.t.}\quad
& v = P\LSE(r+\gamma v).
\qquad {\color{gray}(\text{soft Bellman})}
\end{aligned}
\end{equation}

This problem is highly non-unique. Because the objective is invariant under
potential-based transformations of the reward, the relaxed problem admits an
entire equivalence class of solutions. We will exploit this invariance later to
recover a solution to the normalized problem \eqref{eq:main-irl} from a single
convenient representative.

\subsection{Behavior cloning solves the relaxed problem}
\label{sec:easy-solution}

Without normalization, the relaxed problem admits a particularly simple
solution: estimate the behavior policy \(\pi\)  and set
\(r(s,a)=\log \pi(a\mid s)\) and \(v(s)=0\). This pair maximizes the
conditional log-likelihood and satisfies the soft Bellman equation. Indeed, $P\LSE(r+\gamma v) = 0$ since $\exp \{\LSE(r+\gamma v)(s)\} = \sum_a \exp\{r(s,a)\}
=
\sum_a \pi(a\mid s)
=
1.$
We will use the shorthand
\[
u^\star(s,a):=\log \pi(a\mid s).
\]
\begin{lemma}[Trivial optimum of the relaxed problem]
\label{lem:trivial}
The pair \((r,v) := (u^\star,0)\) solves \eqref{eq:relaxed-irl}.
\end{lemma}

A related observation appears in Section~4 of \citet{fu2018airl}, but not as a
tool for identification or estimation. Here, by contrast, it is the starting
point for normalization-based identification. As shown in
Section~\ref{sec:aside}, this trivial solution already suffices for policy-value
comparisons.

\subsection{An invariance among solutions}
\label{sec:invariance}
The relaxed problem is invariant to \emph{potential-based shaping}: adding a state-only potential $c:\mathcal{S} \to \mathbb{R}$ shifts all logits in the softmax by the same amount per state, leaving both feasibility and likelihood unchanged. This is the entropy-regularized analogue of reward shaping in classical RL \citep{ng1999policy} and explains why the relaxed objective is flat along an affine subspace.  

\begin{lemma}[Potential-based shaping invariance]
\label{lem:shaping}
Let \((r,v)\) be feasible for \eqref{eq:relaxed-irl}, and let
\(c:\mathcal S\to\mathbb R\) be bounded. Define
\(\tilde r = r + c - \gamma Pc\) and \(\tilde v = v + Pc\).
Then \((\tilde r,\tilde v)\) is also feasible for \eqref{eq:relaxed-irl} and
attains the same objective value as \((r,v)\). In particular,
\(r+\gamma v-\LSE(r+\gamma v)\), the induced log-policy, is unchanged.
\end{lemma}

Related partial-identification results appear in
\citet[Theorem~1]{cao2021identifiability}, \citet{fu2018airl}, and
\citet[Lemma~2]{geng2020deep}. We next show that our normalization
selects a unique representative from this class.
 
\subsection{Solving the original normalized problem}
\label{sec:solve-normalization}

Lemma~\ref{lem:trivial} gives one relaxed optimum, \((u^\star,0)\), and
Lemma~\ref{lem:shaping} characterizes all others via potential-based
transformations of the form \((r,v) = (u^\star + c - \gamma Pc, Pc)\). Since
the constrained and relaxed problems have the same optimal value, solving
\eqref{eq:main-irl} amounts to finding the shaping function \(c\) such that
the corresponding pair \((r,v)\) satisfies the desired constraint. For our
normalization \(\mu r = g\), \(c\) is uniquely determined by a Bellman
equation.

The next result gives the corresponding solution in terms of
\(Q^\mu_{u^\star-g}\), the \(Q\)-function under reward \(u^\star-g\) and policy
\(\mu\), where \(u^\star=\log \pi\). Recall that
\[
(P\mu Q)(s,a):=\mathbb{E}_{s' \sim P(\cdot\mid s,a),\, a' \sim \mu(\cdot\mid s')}
\bigl[Q(s',a')\bigr].
\]

\begin{theorem}[IRL via a Bellman equation]
\label{thm:unique}
Let \(Q^\mu_{u^\star-g}\) be the unique bounded solution in
\(L^\infty(\nu_\pi)\) to
\[
Q^\mu_{u^\star-g}(s,a) = u^\star(s,a)-g(s) + \gamma (P \mu Q^\mu_{u^\star-g})(s,a).
\]
Then \eqref{eq:main-irl} admits a unique optimal solution in
\(L^\infty(\nu_\pi)\). For \(\rho\)-almost every \(s\) and every action \(a\),
it is given by
\begin{align*}
r^\star(s,a)
    &= Q^\mu_{u^\star-g}(s,a)
    - (\mu Q^\mu_{u^\star-g})(s)
    + g(s), \\
v^\star(s,a)
    &= -(P\mu Q^\mu_{u^\star-g})(s,a).
\end{align*}
When \(\gamma>0\), equivalently
\(v^\star=\gamma^{-1}(u^\star-g-Q^\mu_{u^\star-g})\).
\end{theorem}

Theorem~\ref{thm:unique} reduces normalized reward recovery to two steps: first estimate \(u^\star = \log \pi\), then solve
the linear Bellman equation for \(Q^\mu_{u^\star-g}\). The normalized reward is then obtained in closed form as the advantage function
\(Q^\mu_{u^\star-g} - \mu Q^\mu_{u^\star-g}\), shifted by \(g\). By Lemma~\ref{lem:shaping}, every feasible
representative in the shaping class induces the same \(u^\star\). Therefore,
once one feasible representative is known, we can impose a different
normalization without re-solving the original IRL problem.  

 \textbf{Comparison to DeepPQR.}
In the anchor-action setting \(\mu(a \mid s)=1\{a=a^\dagger\}\), DeepPQR
\citep{geng2020deep} first estimates the behavior policy and then learns the
entropy-regularized soft anchor \(Q\)-function, denoted here by \(Q_A\), which
satisfies
\[
Q_A(s)=g(s)+\gamma\E\!\left[
Q_A(s')-u^\star(s',a^\dagger)\mid s,a^\dagger\right].
\]
This quantity is related to GenPQR's anchor-policy value by
\(W(s)=u^\star(s,a^\dagger)-Q_A(s)\), which obeys
\[
W(s)=u^\star(s,a^\dagger)-g(s)+\gamma \E[W(s')\mid s,a^\dagger].
\]
Thus, DeepPQR estimates \(Q_A\), which can be transformed into \(W\), rather
than estimating \(W\) directly. GenPQR instead works with
\(Q^\mu_{u^\star-g}\) itself, from which the normalized reward and continuation
value \(v^\star\) follow by the normalized Bellman equation. This avoids
separately estimating \(PW\) and gives a modular second stage based on standard
\(Q\)-evaluation.

DeepPQR relies on a
well-supported anchor action \(a^\dagger\). When such observations are limited,
GenPQR can estimate the relevant \(Q\)-function using methods that share
information across actions, and it extends beyond fixed-anchor discrete-action
settings. This is consistent with our experiments: GenPQR is more stable when
anchor actions are rare or the action space is large, whereas DeepPQR's
anchor-action regression degrades in these settings and does not directly extend
to continuous actions.

\subsection{Behavior cloning suffices for policy comparison}
\label{sec:aside}

For policy comparison, exact reward recovery is unnecessary once the transition
kernel \(P\) and discount factor \(\gamma\) are fixed. Let
\(V_r^\pi=\pi Q_r^\pi\) denote the value function under reward \(r\) and policy
\(\pi\), where \(Q_r^\pi\) is the corresponding \(Q\)-function.

\begin{theorem}[Identification of policy value differences]
\label{thm:identification}
Let $(r,v)$ be a bounded solution of \eqref{eq:relaxed-irl}, for example
$(u^\star,0)$. Then,
for any two policies \(\pi_1,\pi_2\),
\[
V_{r^\star}^{\pi_1}(s)-V_{r^\star}^{\pi_2}(s)
=
V_r^{\pi_1}(s)-V_r^{\pi_2}(s)
\qquad\text{for }\rho\text{-almost every }s.
\]
\end{theorem}

Thus, exact reward recovery is needed only for targets that depend on the
normalization itself, such as evaluation under counterfactual transition
dynamics or discount factors, or interpretation of structural features of
\(r\). Otherwise, \(u^\star\), or any reward solving the relaxed problem,
suffices for policy comparison. This parallels \citet{hotz1993conditional},
where under Gumbel shocks value differences are identified from log-odds of
observed choices.

\section{A generic algorithm}
\label{sec:genalg}
Algorithm~\ref{alg:generic-irl} presents Generalized
Policy-to-\(Q\)-to-Reward (GenPQR) for reward recovery. The method has two
standard steps. First, we estimate the behavior policy \(\pi(a \mid s)\),
equivalently \(u^\star(s,a)=\log \pi(a \mid s)\), yielding \(\hat u\).
Second, we recover \((r^\star,v^\star)\) by estimating the \(Q\)-function of
policy \(\mu\) under reward \(\hat u-g\), then applying the plug-in formula in
Theorem~\ref{thm:unique}. Appendix~\ref{app:continuous-actions} gives
the continuous-action version: estimate a conditional log-density, solve the
same \(Q\)-fixed point with \(\mu\)-integrals, and apply the same normalized
advantage formula. When \(\hat \pi(a\mid s)\) is
close to \(0\), \(\hat u\) may be unstable;
in practice, one may clip \(\hat \pi\) away from zero
\citep{ionides2008truncated}. Thus IRL reduces to behavior-policy estimation
followed by offline policy evaluation. The unit-scale Gumbel model fixes the
softmax temperature. With shock scale \(\tau>0\), behavior identifies
\(r/\tau\); a known \(\tau\) rescales the recovered reward, whereas choices
alone do not identify an unknown \(\tau\) \citep[Appendix C.2.3]{geng2020deep}.
Changing \(\tau\) rescales reward and value differences but preserves their
signs and orderings, policy rankings, and the optimal policy; cardinal
interpretation requires an external scale normalization.

The first stage can use any
probabilistic classifier, behavior-cloning, imitation-learning, or IRL method
trained on the observed data, including MaxEnt IRL and adversarial IRL. The
second stage can use a range of existing offline RL methods, including
temporal-difference learning \citep{tsitsiklis1996analysis}, fitted
\(Q\)-evaluation (FQE) \citep{munos2008finite,van2025fitted}, and minimax or
saddle-point \(Q\)-learning methods
\citep{uehara2020minimax,imaizumi2021minimax,xie2020q,xie2021batch}. This
allows each stage to use existing methods for handling misspecification or
distribution shift
\citep{fujimoto2019off,agarwal2021deep,chen2019information,foster2021offline}.
In particular, if an upstream IRL method returns a feasible reward--value pair
under one normalization, GenPQR can use the induced policy as input and recover
the corresponding reward under another normalization, without re-solving the
original IRL objective.

\begin{figure}[t]
\centering
\small

\begin{minipage}{\linewidth}
\begin{algorithm}[H]
\caption{\textsc{Generalized Policy-to-\(Q\)-to-Reward (GenPQR)}}
\label{alg:generic-irl}
\begin{algorithmic}[1]
\INPUT Data \(\{(s_i,a_i,s_i')\}_{i=1}^n\), normalization \((\mu,g)\),
discount \(\gamma\)
\STATE \textbf{Policy estimation:} fit
\(\hat u(s,a)\approx \log \pi(a\mid s)\)
\STATE \textbf{\(Q\)-evaluation:} estimate the fixed point
\STATE \(\hat Q(s,a)\approx \hat u(s,a)-g(s)+\gamma(P\mu\hat Q)(s,a)\)
\OUTPUT \(\hat r(s,a)=\hat Q(s,a)-(\mu\hat Q)(s)+g(s)\)
\end{algorithmic}
\end{algorithm}
\end{minipage}
\par\vspace{1em}
\begin{minipage}{\linewidth}
\begin{algorithm}[H]
\caption{\textsc{Fitted $Q$ Evaluation} for \(Q^\mu_{\hat u-g}\)}
\label{alg:simple-irl}
\begin{algorithmic}[1]
\INPUT Data \(\{(s_i,a_i,s_i')\}_{i=1}^n\), log-policy \(\hat u\),
normalization \((\mu,g)\)
\INPUT Discount \(\gamma\), class \(\mathcal F\), iterations \(K\)
\STATE Initialize \(\hat Q^{(0)}(s,a)\gets 0\)
\FOR{\(k=1,\ldots,K\)}
    \STATE For each \(i\), set
    \[
    y_i \gets \hat u(s_i,a_i)-g(s_i)
    +\gamma\sum_{a'}\mu(a'\mid s_i')\hat Q^{(k-1)}(s_i',a')
    \]
    \STATE Fit \(\hat Q^{(k)} \in \mathcal{F}\) by regressing \(y_i\) on
    \((s_i,a_i)\)
\ENDFOR
\OUTPUT \(\hat Q \gets \hat Q^{(K)}\)
\end{algorithmic}
\end{algorithm}
\end{minipage}

\end{figure}

Algorithm~\ref{alg:simple-irl} instantiates the second stage of GenPQR using
FQE, thereby generalizing DeepPQR \citep{geng2020deep}. It requires only
standard tools for policy estimation and regression. Moreover, the regression
step in Algorithm~\ref{alg:simple-irl} need not be solved exactly at each
iteration: one may instead parametrize \(Q\) by a neural network and take one
or a few stochastic-gradient steps per iteration, as in deep fitted
\(Q\)-learning or boosting \citep{riedmiller2005neural,tosatto2017boosted}. The expectation under
\(\mu\) can also be approximated by Monte Carlo sampling to avoid explicit
integration.

\section{Theoretical Guarantees}
\label{sec:theory}

\subsection{A generic and modular deterministic bound}

We now turn to finite-sample guarantees for $\hat r = \hat Q -\mu \hat Q + g$ obtained via Algorithm \ref{alg:generic-irl}. The following result shows how errors
in the estimated log-policy \(\hat u\) and anchored \(Q\)-function \(\hat Q\)
propagate to the recovered reward. We measure errors in the behavior norm \(L^2(\nu_\pi)\), but the Bellman
operator is most naturally analyzed in \(L^2(d_\mu)\), where \(d_\mu\) is any
stationary state distribution under \((\mu,P)\)
\citep{patterson2022generalized,van2025fitted}, that is,
\[
d_\mu = d_\mu P_\mu,
\qquad
P_\mu(s' \mid s) := \int_{\mathcal A} P(s' \mid s,a)\,\mu(a\mid s)\,da.
\]
Recalling that \(\nu_\pi := \rho \otimes \pi\), we write
\[
\|f\|_{2,\mathrm{beh}}
:=
\Bigl(\E_{(s,a)\sim \nu_\pi}[f(s,a)^2]\Bigr)^{1/2},
\qquad
\|f\|_\infty
:=
\operatorname*{ess\,sup}_{s,a} |f(s,a)|.
\]

We impose a coverage assumption requiring the observed system
\((\rho,\pi,P)\) to provide sufficient support relative to the normalization
policy \(\mu\) under the same dynamics \(P\): \(\pi\) must cover \(\mu\), the
behavior state distribution \(\rho\) must cover the stationary distribution
\(d_\mu\), and \(d_\mu\) must cover the one-step state distribution
\(\nu_\pi P\).
 
\begin{assumption}[Policy mismatch, stationary-state coverage, and one-step state coverage]
\label{cond::coverage}
Assume
\[
C_{\mathrm{cov}}
:=
C_{\nu_\pi P/d_\mu}\,C_{d_\mu/\rho}\,C_{\mu/\pi}
< \infty,
\]
where
\[
C_{\mu/\pi}
:=
\sup_{s,a}\frac{\mu(a\mid s)}{\pi(a\mid s)},
\quad
C_{d_\mu/\rho}
:=
\left\|\frac{d d_\mu}{d\rho}\right\|_{\infty},
\quad
C_{\nu_\pi P/d_\mu}
:=
\left\|\frac{d(\nu_\pi P)}{d d_\mu}\right\|_{\infty}.
\]
\end{assumption}
The coverage constant \(C_{\mathrm{cov}}\) is analogous to standard coverage
and concentrability coefficients in offline RL
\citep{xie2022role,zhan2022offline}. It equals \(1\) under stationary sampling
from \((P,\mu)\). The normalization policy $\mu$ can be chosen to favor actions well
supported by $\pi$; an anchor action, for example, should have adequate
support in the behavior data.

\begin{theorem}[Reward recovery bound]
\label{thm:reward-recovery}
Under Assumption~\ref{cond::coverage},
\[
\|r^\star-\hat r\|_{2,\mathrm{beh}}
\le
\sqrt{C_{\mu/\pi}}
\left\{
\|Q^\mu_{\hat u-g}-\hat Q\|_{2,\mathrm{beh}}
+
\left(
1+\frac{\gamma\sqrt{C_{\mathrm{cov}}}}{1-\gamma}
\right)
\|u^\star-\hat u\|_{2,\mathrm{beh}}
\right\}.
\]
\end{theorem}

The reward-estimation error decomposes into a \(Q\)-estimation term,
\(\|Q^\mu_{\hat u-g}-\hat Q\|_{2,\mathrm{beh}}\), and a policy-estimation
term, \(\|\hat u-u^\star\|_{2,\mathrm{beh}}\). This makes the bound modular:
any guarantees for \(Q\)-learning and policy estimation translate directly into
a reward-recovery guarantee, with the latter scaled by \((1-\gamma)^{-1}\) and
the coverage coefficients in Assumption~\ref{cond::coverage}. Under sample
splitting \citep{foster2023orthogonal}, for example, existing results apply
directly to \(\hat u\) and \(\hat Q\), including bounds for FQE
\citep{munos2008finite,van2025fitted} and minimax \(Q\)-learning
\citep{uehara2020minimax}. In the next section, we apply this
template to derive finite-sample bounds for PQR with FQE.

\subsection{Generalization conditions for GenPQR with FQE}
\label{sec:theory2}

We now specialize Theorem~\ref{thm:reward-recovery} to the case in which
\(\hat Q\) is obtained by running \(K\) steps of FQE over a regression class
$\mathcal F$, initialized at $\hat Q^{(0)}\in\mathcal F$, as in
Algorithm~\ref{alg:simple-irl}. To state the resulting finite-sample bound, we
assume PAC-style generalization guarantees for the policy-estimation and
regression steps.
\begin{assumption}[Policy generalization]
\label{assump:policy-gen}
There exists \(\underline p>0\) such that, for all \(\delta\in(0,0.5)\), with
probability at least \(1-\delta\),
\[
\min\{\pi(a\mid s),\hat\pi(a\mid s)\}\ge \underline p
\quad\text{for all }(s,a),
\qquad
\left\{
\E_{s\sim\rho}\!\left[
\mathrm{KL}\!\left(\pi(\cdot\mid s)\,\|\,\hat\pi(\cdot\mid s)\right)
\right]
\right\}^{1/2}
\le
\rho_\pi(n,\delta).
\]
\end{assumption}
Lower bounds on \(\hat \pi\) and \(\pi\) ensure that KL divergence controls \(\|\hat u-u^\star\|_{2,\mathrm{beh}}\). For maximum likelihood estimation over a class \(\mathcal U\), one typically has \(\rho_\pi(n,\delta)\lesssim r_{\mathcal U}(n)+\sqrt{\log(1/\delta)/n},\) where \(r_{\mathcal U}(n)\) is a critical radius determined by the local complexity of \(\mathcal U\)  \citep{geer2000empirical,wainwright2019high}.

We require each regression step to achieve small excess risk relative to the best approximation in \(\mathcal F\) to the Bellman target \(\mathcal T_{\hat u}^\mu(\hat Q^{(k-1)})\), where \(\mathcal T_u^\mu f:=u-g+\gamma P\mu f\). Define
\[
\mathrm{reg}(Q\mid u,Q')
:=
\|\mathcal T_u^\mu(Q')-Q\|_{2,\mathrm{beh}}^2
-
\inf_{f\in\mathcal F}\|\mathcal T_u^\mu(Q')-f\|_{2,\mathrm{beh}}^2.
\]

\begin{assumption}[One-step regression generalization]
\label{assump:q-gen}
Assume \(0\in\mathcal F\) and
\(\sup_{f \in \mathcal{F}} \|f\|_{\infty} < \infty\). For every
\(k\in[K]\) and \(\delta\in(0,0.5)\), with probability at least
\(1-\delta\), \(\{\mathrm{reg}(\hat Q^{(k)}\mid \hat u,\hat Q^{(k-1)})\}^{1/2}
\le
\rho_{Q}(n,\delta)\).
\end{assumption}
Such bounds for nuisance-dependent regression targets follow from
\citet{foster2023orthogonal,van2026researcher}. Under a fresh-sample analysis
\citep{munos2008finite}, one typically has
\(\rho_{Q}(n,\delta)\lesssim r_{\mathcal F}(n/K)+
\sqrt{K\log(1/\delta)/n}\), where \(r_{\mathcal F}(n/K)\) is the local
Rademacher critical radius of \(\mathcal F\). Similar bounds hold for
weakly dependent trajectory data under mixing conditions
\citep{yu1994rates,mohri2010stability}.

\subsection{Finite-sample bound for GenPQR with FQE}
\label{sec:theory3}

We now state our main result for GenPQR with FQE under approximate Bellman-completeness. To account for misspecification in the regression steps, we introduce the
\emph{inherent Bellman error} of the regression class \citep{munos2008finite}:
\[
\varepsilon_{\mathcal F}
:=
\sup_{f \in \mathcal F}\inf_{g \in \mathcal F}
\|\mathcal T_{u^\star}^\mu f-g\|_{\infty}.
\]
This quantity is zero when \(\mathcal F\) is Bellman complete, that is, when
\(\mathcal T_{u^\star}^\mu f \in \mathcal F\) for every \(f \in \mathcal F\)
\citep{chen2019information}. In that case, the targets under $u^\star$ lie in $\mathcal F$.
The targets under $\hat u$ differ by $\hat u-u^\star$; this additional
approximation error is included in the bound below. We define the standard concentrability coefficient
\[
C_{\mathrm{conc}}
:=
\sup_{m\ge 0}
\left\|
\frac{d\bigl(\nu_\pi (P\mu)^m\bigr)}{d\nu_\pi}
\right\|_\infty < \infty,
\]
which is finite under  Assumption~\ref{cond::coverage} with $C_{\mathrm{conc}}
\le C_{\mathrm{cov}}$ by Lemma~\ref{lem:coverage-implies-conc} in Appendix \ref{appendix:lemmas}. 
Set
\[
\kappa_{\mathrm{conc}}
:=
1+\frac{\gamma\sqrt{C_{\mathrm{conc}}}}{1-\gamma}.
\]

\begin{theorem}[Finite-sample bound for GenPQR with FQE]
\label{thm:pqr-fqe}
Assume Assumptions~\ref{cond::coverage}, \ref{assump:policy-gen}, and
\ref{assump:q-gen}.
Then with probability at least \(1-\delta\),
\begin{align*}
\|r^\star-\hat r\|_{2,\mathrm{beh}}
&\le
\sqrt{C_{\mu/\pi}}
\Bigg\{
\sqrt{C_{\mathrm{conc}}}\,
\gamma^K \|\hat Q^{(0)}-Q^\mu_{\hat u-g}\|_{\infty}
+
\kappa_{\mathrm{conc}}
\Bigl(
\varepsilon_{\mathcal F}
+\rho_Q\bigl(n,\tfrac{\delta}{2K}\bigr)
\Bigr)
\\
&\qquad\qquad+
\sqrt{\frac{2}{\underline p}}\left(
1
+\kappa_{\mathrm{conc}}
+\frac{\gamma\sqrt{C_{\mathrm{cov}}}}{1-\gamma}
\right)
\rho_\pi\bigl(n,\tfrac{\delta}{2}\bigr)
\Bigg\}.
\end{align*}
\end{theorem}
\textbf{Proof sketch.}
Apply Theorem~\ref{thm:reward-recovery} and bound
\(\|Q^\mu_{\hat u-g}-\hat Q\|_{2,\mathrm{beh}}\) by the inexact Picard argument
of \citet{munos2008finite}. The only change is that inherent Bellman error is
measured for \(\mathcal T^\mu_{u^\star}\), not the data-dependent
\(\mathcal T^\mu_{\hat u}\); Lemma~\ref{lemma::periter} shows this suffices up
to policy-estimation error. Approximate Bellman completeness could be
relaxed using suitable weighting \citep{van2025fitted} or minimax formulations
\citep{uehara2023offline}.

\textbf{Discussion.}
The bound matches the usual FQE structure
\citep{munos2008finite,van2025fitted}, with additional terms due to estimation
of \(u^\star\). The three terms respectively capture finite-iteration error,
approximation and statistical error from the fitted Bellman updates, and the
IRL-specific error from estimating \(u^\star\). The last term can be large when
the behavior policy is nearly deterministic, for example in low-temperature
softmax regimes, because KL divergence controls \(L^2\) log-policy error with
a factor of \(\underline p^{-1/2}\). In particular, suppose that policy estimation and regression over classes
with pseudo-dimensions $d_{\mathcal U}$ and $d_{\mathcal F}$ admit rates
$\rho_\pi(n,\delta)=\widetilde O(\sqrt{d_{\mathcal U}/n})$ and
$\rho_Q(n,\delta)=\widetilde O(\sqrt{d_{\mathcal F}/n})$, respectively,
where $\widetilde O$ suppresses logarithmic factors. Taking
$K\asymp\log n$ with a sufficiently large constant makes the bounded
initialization error negligible and gives, up to these logarithmic factors,
\[
\|\hat r-r^\star\|_{2,\mathrm{beh}}
\lesssim
\frac{\sqrt{C_{\mu/\pi}C_{\mathrm{cov}}}}{1-\gamma}
\left(
\varepsilon_{\mathcal F}
+
\sqrt{\frac{d_{\mathcal F}}{n}}
+
\frac{1}{\sqrt{\underline p}}\sqrt{\frac{d_{\mathcal U}}{n}}
\right).
\]

If the regression samples are instead split across $K$ iterations, the
regression rate uses $n/K$ in place of $n$.

\textbf{Comparison to DeepPQR.}
Theorem~2 of \citet{geng2020deep} gives a related FQE bound for anchor-action
reward recovery under a specific neural-network architecture. Our result is
more modular: it is not tied to a particular function class, allows approximate
Bellman completeness for \(\mathcal T^\mu_{u^\star}\), makes coverage explicit,
and replaces sup-norm policy-error control with Kullback--Leibler or
\(L^2\)-type control, which is natural for likelihood-based policy estimators,
multiclass classification, MaxEnt IRL \citep{ziebart2008maximum}, and
adversarial IRL \citep{fu2018airl,snoswell2020revisiting,ke2020imitation,
foster2024behavior}.

\section{Experimental investigation}
\label{sec:experiments}

We use DeepPQR's synthetic infinite-horizon environment, with continuous
states, deterministic transitions, and a zero-reward anchor action.
We first compare DeepPQR and GenPQR using the same estimated policy,
then vary the number of actions, and finally examine a setting where
the normalization is estimated from a separate data source.
We evaluate reward estimates against the noise-free reward matrix on
independent test states and report means with \(95\%\) confidence intervals
over \(100\) seeds. Full details of the simulation design and estimators
are given in Appendix~\ref{app:simulation-details}.

\begin{figure*}[htb]
\centering
\begin{minipage}[c]{0.71\textwidth}
\centering
\scriptsize
\resizebox{\linewidth}{!}{%
\begin{tabular}{llcccc}
\toprule
Setting & Method & Anchor ct. (frac.) & MSE $\downarrow$ & Corr. $\uparrow$ & Time $\downarrow$ \\
\midrule
200 / rare & DeepPQR & 311 (0.16) & 2.60 $\pm$ 0.15 & 0.46 $\pm$ 0.03 & 7.33 $\pm$ 0.10 \\
200 / rare & GenPQR & 311 (0.16) & \textbf{0.91 $\pm$ 0.11} & \textbf{0.66 $\pm$ 0.03} & \textbf{5.22 $\pm$ 0.08} \\
\midrule
1000 / rare & DeepPQR & 1559 (0.16) & 1.48 $\pm$ 0.10 & 0.61 $\pm$ 0.03 & 38.84 $\pm$ 0.60 \\
1000 / rare & GenPQR & 1559 (0.16) & \textbf{0.76 $\pm$ 0.08} & \textbf{0.73 $\pm$ 0.02} & \textbf{27.76 $\pm$ 0.42} \\
\midrule
2500 / common & DeepPQR & 6280 (0.25) & 0.70 $\pm$ 0.07 & 0.73 $\pm$ 0.02 & 92.82 $\pm$ 1.28 \\
2500 / common & GenPQR & 6280 (0.25) & \textbf{0.60 $\pm$ 0.06} & \textbf{0.73 $\pm$ 0.02} & \textbf{64.09 $\pm$ 0.90} \\
\bottomrule
\end{tabular}%
}
\end{minipage}
\hfill
\begin{minipage}[c]{0.25\textwidth}
\centering
\includegraphics[width=\linewidth]{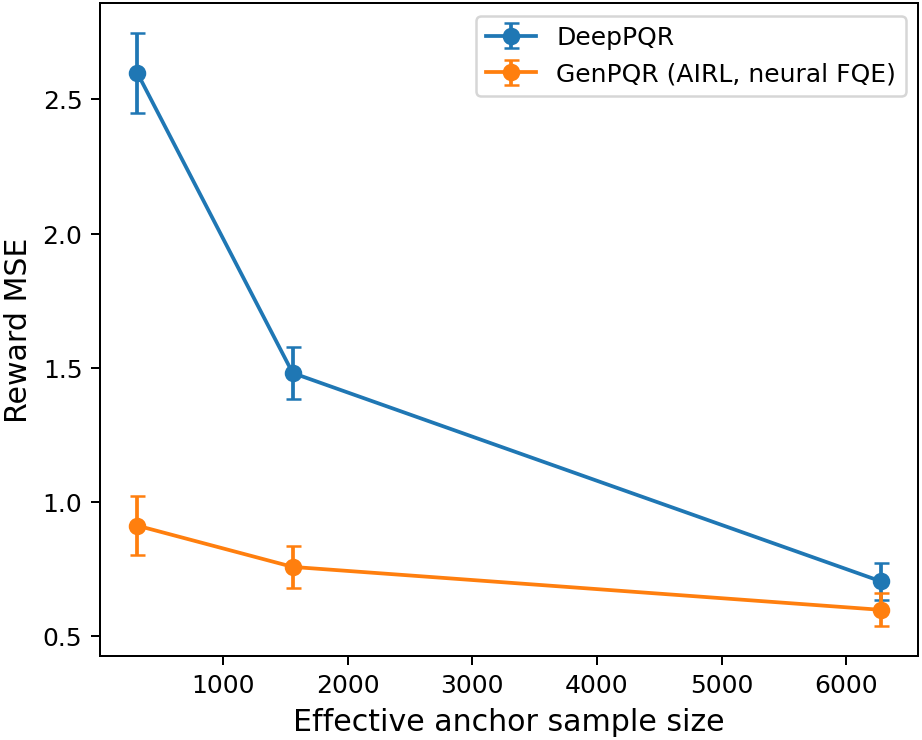}
\end{minipage}
\caption{Comparison of DeepPQR and GenPQR using the same AIRL policy
estimate and neural FQE. Entries are mean \(\pm\) 95\% CI over 100 seeds.
Anchor count is the number of anchor-action transitions used by DeepPQR's
anchor-\(Q\) step.}
\label{fig:exp1_combo}
\end{figure*}

\textbf{Comparison with DeepPQR.}
Both methods use the same AIRL policy estimate and neural function
approximation, allowing us to focus on differences in reward recovery
(Section~\ref{sec:solve-normalization}). We vary the number of trajectories
and the frequency of the anchor action, which determine the sample size
available for DeepPQR's anchor-\(Q\) estimation.
Figure~\ref{fig:exp1_combo} shows results for three settings; Appendix
\ref{app:additional-experiments} reports the full comparison.

\vspace{0.35em}
\noindent\textbf{Scaling the number of actions.}
Holding the behavior sample fixed at \(2500\) trajectories, we vary
\(|\mathcal A|\in\{5,10,20,40\}\). As the number of actions increases,
the observed anchor count falls from about \(6400\) to \(800\).
Figure~\ref{fig:many-action-main} in
Appendix~\ref{app:additional-experiments} shows that DeepPQR's reward
MSE rises above \(4\) at \(40\) actions, while both GenPQR variants
remain below \(0.5\). Unlike DeepPQR's anchor-\(Q\) estimation,
GenPQR's FQE uses transitions from all actions.

\vspace{0.35em}
\noindent\textbf{Combining two data sources.}
We next consider \(20\) actions and two independent samples from the same MDP.
The behavior sample contains \(2500\) trajectories of \((S,A,S')\), but no
rewards. A second sample contains
\(n_g\in\{100,250,500,1000,2500,5000\}\)
trajectories of \((S,Y)\), with actions unobserved and sampled under
the same reference and behavior policy \(\mu=\pi\), where
\(\E[Y\mid S=s,A=a]=r(s,a)\).
The second sample identifies
\(g(s)=\E[Y\mid S=s]=\sum_a\pi(a\mid s)r(s,a)\),
while the behavior sample identifies \(\pi\).
Combining the two samples therefore identifies the normalization
pair \((\pi,g)\), which neither sample identifies alone.
We estimate \(\pi\) by behavior cloning and \(g\) by regression, then
apply GenPQR. The true reward also satisfies \(r(s,0)=0\), so DeepPQR
and anchor GenPQR remain identified using the behavior sample alone,
but these anchor-based estimators do not use the additional outcome data.

With \(100\) auxiliary trajectories, fusion and anchor GenPQR have
similar accuracy (paired MSE difference \(0.008\pm0.026\)).
For larger auxiliary samples, fusion has lower MSE.
At \(5000\) trajectories, its behavior-weighted reward MSE is
\(0.210\pm0.026\), compared with
\(0.443\pm0.064\) for anchor GenPQR and \(2.253\pm0.089\) for DeepPQR.
This is a \(52.6\%\) reduction relative to anchor GenPQR, with lower
MSE in all \(100\) paired seeds.

\noindent\textbf{Conclusion.}
GenPQR recovers normalized rewards through policy estimation and
Bellman evaluation, without a separate reward-regression step.
In simulations, GenPQR improves reward recovery when anchor observations are
limited and when auxiliary outcome data are available.

\section*{Acknowledgments}
The authors have no funding to disclose. The authors declare no competing interests.

\bibliographystyle{plainnat}
\bibliography{references}

\section*{Lay summary}
People's choices can reveal what they value, but behavior alone does not
uniquely determine the underlying reward. We show how to recover a well-defined
reward after specifying a transparent normalization. Our method, GenPQR,
reduces the problem to two familiar statistical tasks: classifying which action
an agent takes and using regression to evaluate future outcomes. This makes the
approach compatible with standard machine-learning tools and auxiliary outcome
data. We establish when the normalized reward is identified and bound the
estimation error from each stage. In simulations, GenPQR is more accurate than
an anchor-action method when the number of possible actions grows and when
complementary outcome data are available.

\appendix

 \section{Higher-sample method comparison}
 \label{app:fighigher}
\begin{figure}[htb]
\centering
\scriptsize
\setlength{\tabcolsep}{3pt}
\begin{tabular}{lcccc}
\toprule
Method & MSE $\downarrow$ & Corr. $\uparrow$ & NLL $\downarrow$ & Time $\downarrow$ \\
\midrule
GenPQR (BC, NN) & \textbf{0.37 $\pm$ 0.06} & \textbf{0.81 $\pm$ 0.01} & \textbf{1.546 $\pm$ 0.004} & 4.74 $\pm$ 0.06 \\
GenPQR (AIRL, NN) & 0.63 $\pm$ 0.06 & 0.72 $\pm$ 0.03 & 1.608 $\pm$ 0.006 & 26.43 $\pm$ 0.34 \\
GenPQR (BC, GBT) & 1.07 $\pm$ 0.19 & 0.49 $\pm$ 0.03 & \textbf{1.546 $\pm$ 0.004} & \textbf{2.25 $\pm$ 0.03} \\
GenPQR (AIRL, GBT) & 1.07 $\pm$ 0.19 & 0.41 $\pm$ 0.05 & 1.608 $\pm$ 0.006 & 23.97 $\pm$ 0.32 \\
\midrule
DeepPQR & 0.79 $\pm$ 0.07 & 0.69 $\pm$ 0.02 & 1.608 $\pm$ 0.006 & 38.24 $\pm$ 0.42 \\
MaxEnt-IRL & 1.17 $\pm$ 0.18 & 0.35 $\pm$ 0.03 & 1.624 $\pm$ 0.005 & 8.57 $\pm$ 0.11 \\
AIRL state reward & 2.17 $\pm$ 0.21 & -0.01 $\pm$ 0.03 & 1.608 $\pm$ 0.006 & 23.83 $\pm$ 0.32 \\
Log-policy pseudo & 3.40 $\pm$ 0.18 & 0.37 $\pm$ 0.05 & 1.608 $\pm$ 0.006 & 23.83 $\pm$ 0.32 \\
\bottomrule
\end{tabular}
\caption{Higher-sample method comparison. Entries are mean $\pm$ 95\% confidence interval over 100 seeds. NN = neural FQE; GBT = boosted FQE.}
\label{fig:exp2_combo}
\end{figure}

\section{Abstract linear normalizations.}
The same argument used to derive Theorem \ref{thm:unique} extends to general linear normalizations. Let \(A\) be a
bounded linear operator on reward functions, and impose \(Ar=g\). Every relaxed
optimum inducing the observed policy has the form
\[
    r_c=u^\star+c-\gamma Pc,
    \qquad
    v_c=Pc ,
\]
for some state function \(c\). Thus the normalization \(Ar_c=g\) is equivalent
to
\[
    A(c-\gamma Pc)=g-Au^\star .
\]
If \(g-Au^\star\) lies in the range of the map \(c\mapsto A(c-\gamma Pc)\), and
this map is injective on bounded state functions, then the normalization selects a
unique representative of the reward class:
\[
    c^\star
    =
    \{A(I-\gamma P)\}^{-1}(g-Au^\star),
    \qquad
    r^\star=u^\star+c^\star-\gamma Pc^\star,
    \qquad
    v^\star=Pc^\star .
\]

The statewise normalization above is the special case \(A=\mu\). In this case,
the map becomes \(c\mapsto c-\gamma P_\mu c=(I-\gamma P_\mu)c\), so the
abstract equation reduces to the Bellman resolvent equation in
Theorem~\ref{thm:unique}.

\section{Simulation Details}
\label{app:simulation-details}

The experiment code, pinned environment,
exact commands, tests, and plotting scripts are available in the
\href{https://github.com/Larsvanderlaan/rl-eval-kit/tree/b1cc294487a95f3a588add1e3524ee690c7100e5/submissions/irl/experiments/conference_genpqr/repro}{public
reproduction repository}.

\textbf{Environment.}
We adapt the synthetic environment of \citet{geng2020deep}. States are
continuous with dimension \(p=5\), the action set has \(|\mathcal A|=5\)
actions, and the discount factor is \(\gamma=0.9\). We use an infinite-horizon
stationary data-generating process and generate finite trajectories of length
\(10\) for offline training and evaluation. The normalization is the anchor
constraint \(r(s,a^\dagger)=0\), implemented as anchor-action normalization
with \(a^\dagger=0\) and \(g(s)=0\). Train and test sets within each seed are
generated from the same underlying MDP parameters, with independent rollouts.

\textbf{Behavior policy and transitions.}
The logged policy is generated from a soft \(Q\)-planner with \(\alpha=1\), so
that the DeepPQR log-policy identity is correctly matched. The planner induces
a heterogeneous but non-degenerate behavior policy, and we vary anchor support
by shifting the anchor-action logit. Transitions are deterministic conditional
on state and action up to boundary handling: actions induce action-specific
state shifts, and trajectories that leave the bounded state region are reset
uniformly within that region. We clip estimated action probabilities to
\([0.01,0.99]\) and renormalize before using them in either GenPQR or DeepPQR.

\paragraph{Experiment 1.}
The matched DeepPQR-vs.-GenPQR comparison uses shared AIRL policy estimation
and neural downstream approximation. We consider three settings:
\texttt{(200, -1.0)}, \texttt{(1000, -1.0)}, and \texttt{(2500, 0.0)}, where
the first entry is the number of training trajectories and the second is the
anchor-logit shift. These correspond to low-sample rare-anchor, mid-sample
rare-anchor, and high-sample common-anchor regimes. Each setting uses \(300\)
test trajectories and \(100\) random seeds.

\paragraph{Experiment 2.}
The higher-sample comparison fixes \(1000\) training trajectories, \(300\) test
trajectories, and anchor-logit shift \(0.0\), again over \(100\) seeds. We
compare DeepPQR, GenPQR with two policy estimators and two \(Q\)-estimators,
and the baselines described below.

\paragraph{Experiment 3: action scaling.}
We fix \(2500\) behavior trajectories and \(300\) test trajectories and vary
\(|\mathcal A|\in\{5,10,20,40\}\), using \(100\) seeds at each action count.
The anchor-logit shift remains \(0\). We compare DeepPQR and GenPQR using both
behavior cloning and AIRL policy estimates, with neural downstream estimators.

\paragraph{Experiment 4: Data fusion.}
We fix \(20\) actions, \(2500\) behavior trajectories, and \(300\) independent
test trajectories. For each seed, we generate an independent auxiliary sample
of \(n_g\in\{100,250,500,1000,2500,5000\}\) trajectories of length \(10\).
The auxiliary sample contains only \((S,Y)\), while the behavior sample
contains only \((S,A,S')\). The reference and behavior policies coincide, so
\(g(s)=\E[Y\mid S=s]=\sum_a\pi(a\mid s)r(s,a)\).
We estimate \(\pi\) by behavior cloning and \(g\) by gradient-boosted
regression, using the same estimated policy and neural FQE procedure
for all GenPQR variants. Auxiliary samples are nested by complete
trajectory. We report results over \(100\) new seeds not used during
method development. We also include an oracle variant that uses
the true \((\pi,g)\) but still estimates \(Q\) by FQE, as a benchmark
for the effect of nuisance estimation.

\paragraph{Policy estimators.}
AIRL uses a standard action-independent reward network and potential network,
both implemented as two-layer ReLU MLPs with hidden widths \((64,64)\), Adam
step size \(10^{-3}\), \(60\) behavior-cloning warm-start epochs, and \(80\)
adversarial updates in the paper experiments. Behavior cloning uses the same
\((64,64)\) MLP architecture and \(40\) epochs in Experiment~2. MaxEnt-IRL is
implemented as a neural softmax-\(Q\) model with hidden widths \((128,128)\)
and \(150\) gradient steps.

\paragraph{\(Q\)-evaluation.}
Neural FQE uses a dueling-style MLP \(Q(s,a)=V(s)+A(s,a)-\bar A(s)\) with
hidden widths \((128,128)\), Adam step size \(5\times 10^{-3}\), \(8\) Bellman
iterations, and \(4\) epochs of regression per iteration. Boosted FQE uses
LightGBM directly, with \(4\) outer Bellman iterations and \(30\) boosting
rounds per iteration, learning rate \(0.05\), \(32\) leaves, and minimum leaf
size \(20\).

\paragraph{DeepPQR and baselines.}
DeepPQR uses the same shared AIRL policy estimate as GenPQR in the matched
comparison. It then estimates the anchor \(Q\)-function on the anchor-action
subset, reconstructs the full \(Q\)-function from log-policy ratios, and fits
the final reward-regression network. The AIRL state-reward baseline uses the
AIRL reward head directly; the log-policy pseudo-reward baseline uses
\(\log \hat\pi(a\mid s)-g(s)\); MaxEnt-IRL uses its learned \(Q\)-surrogate as
reward. The linear reward baseline is a supervised reference: it fits
training rewards using a linear state-action model with action-specific
coefficients and ridge regularization.

\paragraph{Reporting.}
For each method we report reward mean squared error, reward correlation with
the ground-truth normalized reward, held-out policy negative log-likelihood
when applicable, and wall-clock runtime. Confidence intervals are normal
approximation intervals based on the \(100\) seed-level estimates. The
data-fusion study reports both behavior-policy-weighted and uniform-over-action
MSE on the full noise-free test reward matrix; no oracle quantity is used to
fit or select an estimated method.

\section{Additional Experimental Results}

\label{app:additional-experiments}

\begin{figure}[H]
\centering
\begin{minipage}[c]{0.48\textwidth}
\centering
\includegraphics[width=\linewidth]{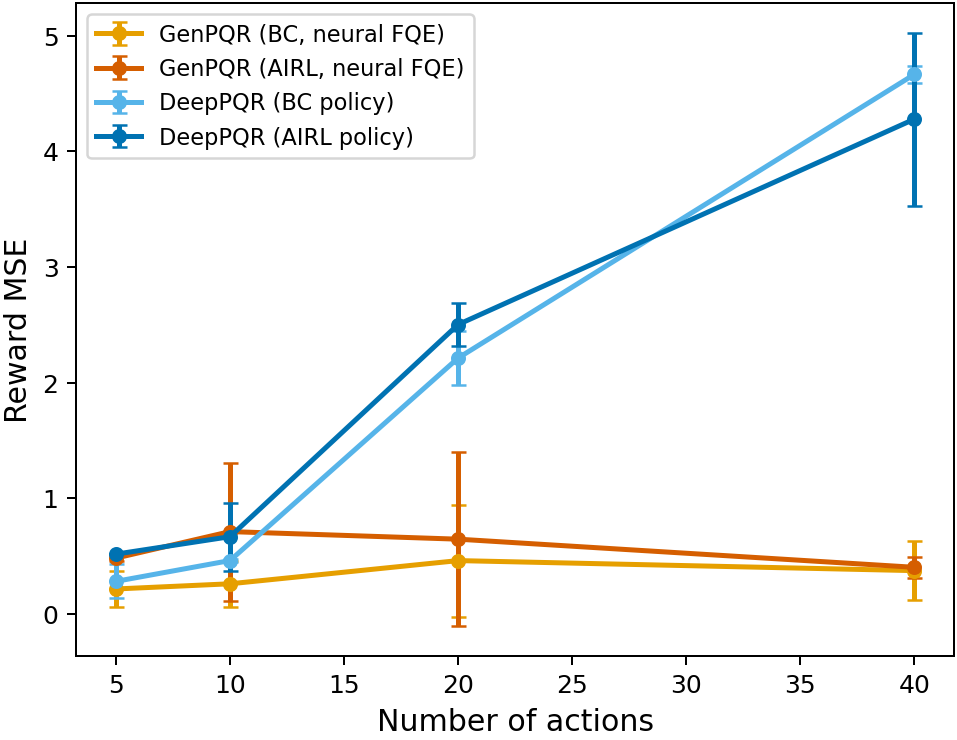}
\end{minipage}
\hfill
\begin{minipage}[c]{0.48\textwidth}
\centering
\includegraphics[width=\linewidth]{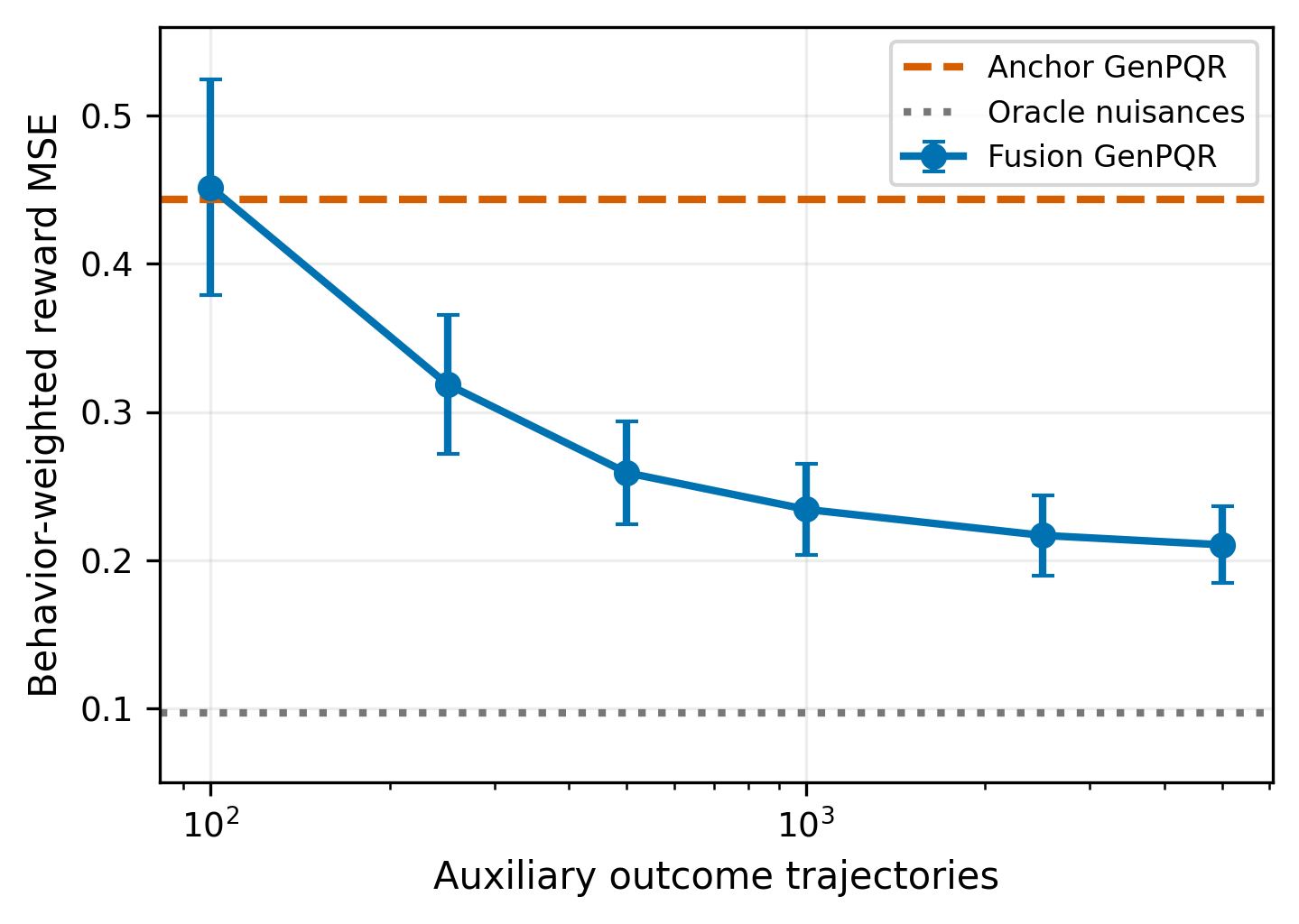}
\end{minipage}
\caption{Large-action experiments. Left: reward MSE as the number of actions
grows at fixed behavior-sample size. Right: 20-action GenPQR using an
independent state--outcome source. Error bars are 95\% confidence intervals
over 100 seeds.}
\label{fig:many-action-main}
\end{figure}

\paragraph{Method comparison.} Figure~\ref{fig:exp2_methods} adds a supervised
linear-reward reference to the higher-sample comparison in
Figure~\ref{fig:exp2_combo}.

\begin{figure}[htb]
\centering
\includegraphics[width=0.95\linewidth]{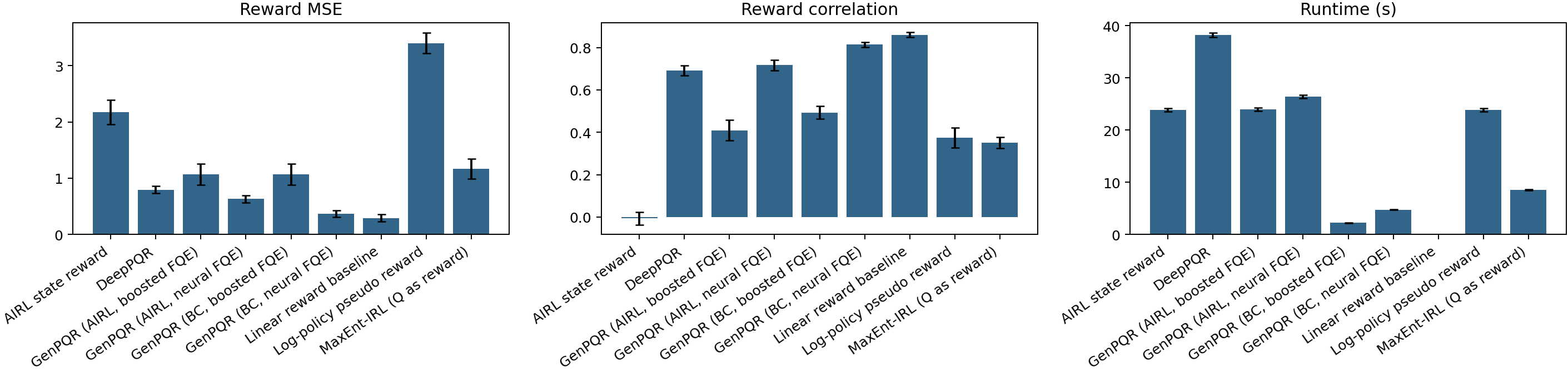}
\caption{Higher-sample comparison of reward MSE, reward correlation, and runtime.
The linear reward baseline is fitted to training rewards. Error bars denote
95\% confidence intervals over 100 seeds.}
\label{fig:exp2_methods}
\end{figure}

\paragraph{Action-support diagnostic.}
Figure~\ref{fig:app-anchor-count} reports the anchor counts underlying the
left panel of Figure~\ref{fig:many-action-main}.

\begin{figure}[htb]
\centering
\includegraphics[width=0.62\linewidth]{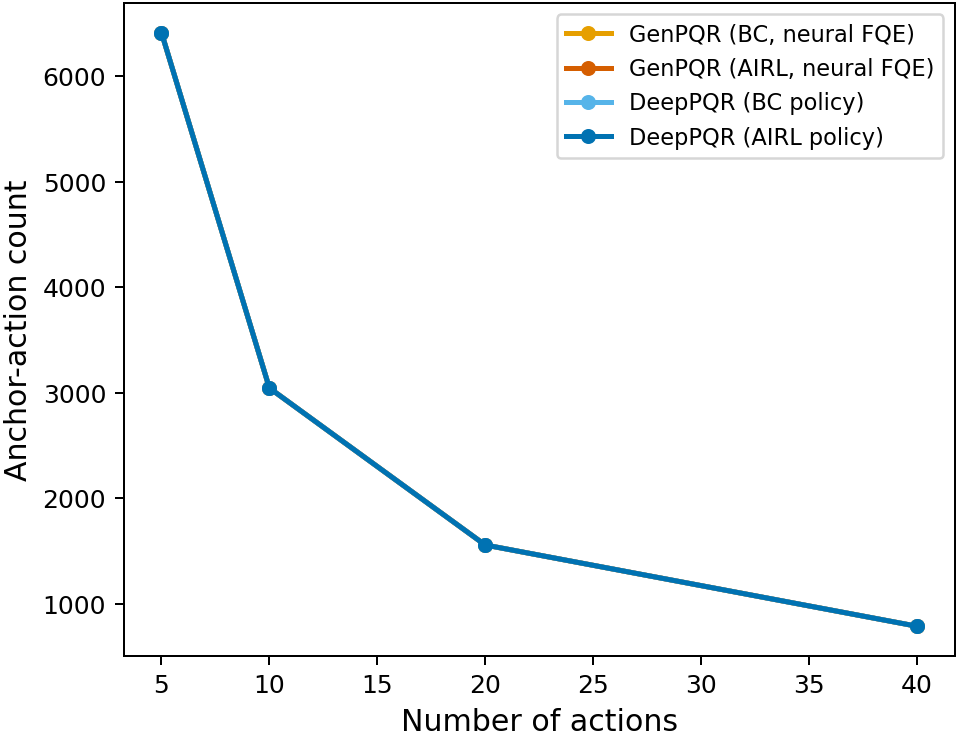}
\caption{Realized anchor-action count as the action set grows with the total
behavior sample held fixed.}
\label{fig:app-anchor-count}
\end{figure}

\paragraph{Complementary-source results.}
Table~\ref{tab:data-fusion-full} gives the exact results underlying
Figure~\ref{fig:many-action-main} (right). The projected DeepPQR row enforces the known
zero anchor on its final predictions; this sensitivity analysis improves its
behavior-weighted MSE by only \(3.5\%\), so anchor violation does not explain
the gap.

\begin{table}[htb]
\centering
\small
\resizebox{\linewidth}{!}{%
\begin{tabular}{lrrrrr}
\toprule
Method & \(n_g\) & Behavior MSE & Uniform MSE & Uniform corr. & \(g\) RMSE \\
\midrule
DeepPQR & -- & \(2.253\pm0.089\) & \(2.247\pm0.089\) & \(0.560\pm0.025\) & -- \\
DeepPQR, anchor projected & -- & \(2.174\pm0.086\) & \(2.187\pm0.087\) & \(0.566\pm0.025\) & -- \\
Anchor GenPQR & -- & \(0.443\pm0.064\) & \(0.454\pm0.064\) & \(0.790\pm0.011\) & -- \\
\midrule
Fusion GenPQR & 100 & \(0.452\pm0.073\) & \(0.471\pm0.074\) & \(0.795\pm0.009\) & \(0.388\pm0.035\) \\
Fusion GenPQR & 250 & \(0.319\pm0.047\) & \(0.337\pm0.048\) & \(0.841\pm0.006\) & \(0.291\pm0.026\) \\
Fusion GenPQR & 500 & \(0.259\pm0.035\) & \(0.277\pm0.036\) & \(0.864\pm0.006\) & \(0.230\pm0.021\) \\
Fusion GenPQR & 1000 & \(0.234\pm0.031\) & \(0.252\pm0.032\) & \(0.876\pm0.006\) & \(0.181\pm0.017\) \\
Fusion GenPQR & 2500 & \(0.217\pm0.027\) & \(0.233\pm0.028\) & \(0.884\pm0.007\) & \(0.137\pm0.013\) \\
Fusion GenPQR & 5000 & \(0.210\pm0.026\) & \(0.227\pm0.027\) & \(0.888\pm0.007\) & \(0.114\pm0.011\) \\
\midrule
GenPQR, oracle nuisances & -- & \(0.097\pm0.017\) & \(0.105\pm0.018\) & \(0.963\pm0.003\) & \(0\) \\
\bottomrule
\end{tabular}%
}
\caption{Twenty-action complementary-source experiment. Entries are mean
\(\pm\) 95\% confidence interval over 100 seeds. The oracle-nuisance row uses the true \(\pi,g\) and estimates \(Q\) by FQE.}
\label{tab:data-fusion-full}
\end{table}

\clearpage
\section{Identification and Representation}
\label{app:identification}

\subsection{Proof of Lemma~\ref{lem:trivial}}

\begin{proof}
Since $\sum_a\exp\{u^\star(s,a)\}=1$, we have $\LSE u^\star=0$ and
$P\LSE u^\star=0$. Thus $(u^\star,0)$ is feasible.
For any feasible $(r,v)$, put $q=r+\gamma v$ and
$\pi_q(a\mid s)=\exp\{q(s,a)-\LSE q(s)\}$. Its objective equals
\[
\E_{\nu_\pi}[q-\LSE q]
=
\E_{\nu_\pi}[u^\star]
-\E_{s\sim\rho}\!\left[
\KL\{\pi(\cdot\mid s)\,\|\,\pi_q(\cdot\mid s)\}\right].
\]
The divergence is nonnegative and is zero exactly when $\pi_q=\pi$
$\rho$-almost everywhere. The pair $(u^\star,0)$ attains this upper bound.
\end{proof}

\subsection{Proof of Lemma~\ref{lem:shaping}}

\begin{proof}
Put $q=r+\gamma v$. By construction,
$\tilde r+\gamma\tilde v=q+c$, and the state-only shift satisfies
$\LSE(q+c)=\LSE q+c$. Hence
\[
P\LSE(\tilde r+\gamma\tilde v)
=P\LSE q+Pc=v+Pc=\tilde v,
\]
so the transformed pair is feasible. Its induced log-policy is
\[
\tilde r+\gamma\tilde v-\LSE(\tilde r+\gamma\tilde v)
=q+c-(\LSE q+c)=q-\LSE q.
\]
The log-policy, and therefore its expected log-likelihood, is unchanged.
\end{proof}

\subsection{Proof of Theorem~\ref{thm:unique}}

\begin{proof}
For finite actions, positivity of $\pi$ and $\nu_\pi P\ll\rho$ imply
$\rho P_\mu\ll\rho$ and $\nu_\pi P\mu\ll\nu_\pi$.
Thus $P\mu$ is well defined on $L^\infty(\nu_\pi)$ and has norm at most one.
Since $u^\star-g$ is bounded and $\gamma<1$, the map
$Q\mapsto u^\star-g+\gamma P\mu Q$ is a contraction and has a unique
bounded fixed point $Q^\mu_{u^\star-g}$.

Define $(r^\star,v^\star)$ as in the theorem. Then $\mu r^\star=g$,
and the fixed-point equation gives
\[
r^\star+\gamma v^\star
=Q^\mu_{u^\star-g}-\mu Q^\mu_{u^\star-g}+g
-\gamma P\mu Q^\mu_{u^\star-g}
=u^\star-\mu Q^\mu_{u^\star-g}.
\]
Because $\LSE u^\star=0$, it follows that
$\LSE(r^\star+\gamma v^\star)=-\mu Q^\mu_{u^\star-g}$.
Consequently, $P\LSE(r^\star+\gamma v^\star)=v^\star$, and the induced
log-policy is $u^\star$. The pair is therefore feasible for the normalized
problem and attains the relaxed optimum from Lemma~\ref{lem:trivial}.

For uniqueness, let $(r,v)$ be any bounded optimal feasible pair and put
$q=r+\gamma v$. Since the normalized and relaxed optimal values coincide,
the equality case in the KL identity from Lemma~\ref{lem:trivial} gives
$q-\LSE q=u^\star$. With the bounded state function $c=\LSE q$, feasibility
and normalization yield
\[
v=Pc,\qquad r=u^\star+c-\gamma Pc,\qquad
(I-\gamma P_\mu)c=g-\mu u^\star.
\]
The Markov operator $P_\mu$ has norm at most one on $L^\infty(\rho)$, so
\[
c=(I-\gamma P_\mu)^{-1}(g-\mu u^\star)
=\sum_{j=0}^\infty\gamma^jP_\mu^j(g-\mu u^\star)
\]
is uniquely determined. Hence so are $r$ and $v$. Finally, the fixed-point
equation gives $\gamma v^\star=u^\star-g-Q^\mu_{u^\star-g}$, proving the
alternative formula when $\gamma>0$.
\end{proof}

\subsection{An Equivalent Identification}

The next result rewrites the identification in terms of the continuation value
\(v^\star\). It is equivalent to Theorem~\ref{thm:unique}, but it enables direct
modeling of \(v^\star\), from which the reward \(r^\star\) is obtained
immediately.

\begin{theorem}[An equivalent identification with anchor function]
\label{thm:unique2}
\Cref{eq:main-irl} admits a unique optimal solution in
\(L^\infty(\nu_\pi)\), where
\[
r^\star
=
\ustar + \mu\!\left(g+\gamma v^\star-\ustar\right)-\gamma v^\star,
\]
and \(v^\star\) is the unique bounded solution to the fixed-point equation
\[
v^\star = P\mu\!\left(g+\gamma v^\star-\ustar\right).
\]
\end{theorem}

\begin{proof}
By Theorem~\ref{thm:unique},
$v^\star=-P\mu Q^\mu_{u^\star-g}$ and
$Q^\mu_{u^\star-g}=u^\star-g-\gamma v^\star$. Substitution gives
\[
v^\star=P\mu(g+\gamma v^\star-u^\star),
\qquad
r^\star=u^\star+\mu(g+\gamma v^\star-u^\star)-\gamma v^\star.
\]
The map $v\mapsto P\mu(g+\gamma v-u^\star)$ is a $\gamma$-contraction in
$L^\infty(\nu_\pi)$, so this fixed-point equation has no other bounded
solution. Uniqueness of the optimal pair follows from Theorem~\ref{thm:unique}.
\end{proof}

\subsection{Proof of Theorem~\ref{thm:identification}}

\begin{proof}
The equality case in Lemma~\ref{lem:trivial}, together with Bellman
feasibility, gives
$r=u^\star+c-\gamma Pc$ and $v=Pc$, where
$c=\LSE(r+\gamma v)$ is bounded. For any policy $\eta$, the Bellman equation
for $Q_{u^\star}^\eta$ implies
\begin{align*}
r+\gamma P\eta(Q_{u^\star}^\eta+c)
&=u^\star+c-\gamma Pc+\gamma P\eta Q_{u^\star}^\eta+\gamma Pc\\
&=Q_{u^\star}^\eta+c.
\end{align*}
Thus uniqueness of the bounded Bellman solution gives
$Q_r^\eta=Q_{u^\star}^\eta+c$ and
$V_r^\eta=V_{u^\star}^\eta+c$.
The normalized solution similarly satisfies
$V_{r^\star}^\eta=V_{u^\star}^\eta+c^\star$, with
$c^\star=-\mu Q^\mu_{u^\star-g}$ by Theorem~\ref{thm:unique}.
Applying these identities to $\eta=\pi_1,\pi_2$ cancels the state-only shifts:
\[
V_{r^\star}^{\pi_1}-V_{r^\star}^{\pi_2}
=V_{u^\star}^{\pi_1}-V_{u^\star}^{\pi_2}
=V_r^{\pi_1}-V_r^{\pi_2}.
\qedhere
\]
\end{proof}

\section{Reward Recovery and Finite-Sample Analysis}
\label{app:recovery}

\subsection{Proof of Theorem~\ref{thm:reward-recovery}}

\begin{proof}
For any measure \(d\) on \(\mathcal S\), let
\(
\|h\|_{2,d} := \{\E_{s\sim d}[h(s)^2]\}^{1/2}.
\)
Since
\(
r^\star(s,a)=Q^\mu_{u^\star-g}(s,a)-(\mu Q^\mu_{u^\star-g})(s)+g(s),
\)
write
\(
\widetilde\Delta_Q:=Q^\mu_{u^\star-g}-\hat Q.
\)
Then \(r^\star-\hat r=\widetilde\Delta_Q-\mu\widetilde\Delta_Q\).
For a state--action function \(f\), define
\[
w_s(a):=\frac{\mu(a\mid s)}{\pi(a\mid s)},
\qquad
\bar f_\pi(s):=\sum_a\pi(a\mid s)f(s,a),
\qquad
h(s,a):=f(s,a)-\bar f_\pi(s).
\]
Using the \(L^2\{\pi(\cdot\mid s)\}\) inner product,
\(h\) is orthogonal to constants and
\[
f-\mu f=h-\langle w_s,h\rangle_\pi.
\]
Moreover, \(\langle w_s,1\rangle_\pi=1\), so
\(1+\|w_s-1\|_{2,\pi(\cdot\mid s)}^2
=\|w_s\|_{2,\pi(\cdot\mid s)}^2\).
Therefore
\begin{align*}
\|f-\mu f\|_{2,\pi(\cdot\mid s)}^2
&=
\|h\|_{2,\pi(\cdot\mid s)}^2
+\langle w_s-1,h\rangle_\pi^2 \\
&\le
\|w_s\|_{2,\pi(\cdot\mid s)}^2
\|h\|_{2,\pi(\cdot\mid s)}^2 \\
&\le
C_{\mu/\pi}\|f\|_{2,\pi(\cdot\mid s)}^2,
\end{align*}
where
\(
\|w_s\|_{2,\pi(\cdot\mid s)}^2
=\sum_a\mu(a\mid s)^2/\pi(a\mid s)
\le C_{\mu/\pi}.
\)
Integrating over \(s\sim\rho\) and taking
\(f=\widetilde\Delta_Q\) yields
\begin{equation}
\label{eq:reward-to-Q}
\|r^\star-\hat r\|_{2,\mathrm{beh}}
\le
\sqrt{C_{\mu/\pi}}\|\widetilde\Delta_Q\|_{2,\mathrm{beh}}.
\end{equation}

Now let
\(
\Delta Q:=Q^\mu_{u^\star-g}-Q^\mu_{\hat u-g}
\)
and
\(
\Delta u:=u^\star-\hat u.
\)
Then
\begin{equation}
\label{eq:Q-triangle}
\|\widetilde\Delta_Q\|_{2,\mathrm{beh}}
\le
\|Q^\mu_{\hat u-g}-\hat Q\|_{2,\mathrm{beh}}
+
\|\Delta Q\|_{2,\mathrm{beh}}.
\end{equation}
Subtracting the Bellman equations gives
\(
\Delta Q=\Delta u+\gamma P\mu\,\Delta Q,
\)
hence
\(
\|\Delta Q\|_{2,\mathrm{beh}}
\le
\|\Delta u\|_{2,\mathrm{beh}}+\gamma\|P\mu\,\Delta Q\|_{2,\mathrm{beh}}.
\)

By Jensen's inequality, \((P\mu \Delta Q)^2 \le P\mu((\Delta Q)^2)\), so
\begin{align}
\|P\mu\,\Delta Q\|_{2,\mathrm{beh}}^2
&\le
\int P\mu\bigl((\Delta Q)^2\bigr)\,d\nu_\pi
=
\int (\Delta Q)^2\,d(\nu_\pi P\mu) \notag\\
&=
\int \sum_a (\Delta Q(s,a))^2\mu(a\mid s)\,d(\nu_\pi P)(s) \notag\\
&\le
C_{\nu_\pi P/d_\mu}
\int \sum_a (\Delta Q(s,a))^2\mu(a\mid s)\,dd_\mu(s).
\label{eq:pmu-mixed}
\end{align}
Define the mixed norm
\[
\|f\|_{2,d_\mu,\mu}
:=
\left\{
\int \sum_a f(s,a)^2\mu(a\mid s)\,d_\mu(s)
\right\}^{1/2}.
\]
Then \eqref{eq:pmu-mixed} implies
\begin{equation}
\label{eq:DeltaQ-beh}
\|\Delta Q\|_{2,\mathrm{beh}}
\le
\|\Delta u\|_{2,\mathrm{beh}}
+
\gamma\sqrt{C_{\nu_\pi P/d_\mu}}\,
\|\Delta Q\|_{2,d_\mu,\mu}.
\end{equation}

It remains to control \(\|\Delta Q\|_{2,d_\mu,\mu}\). Taking \(\|\cdot\|_{2,d_\mu,\mu}\) norms in
\(
\Delta Q=\Delta u+\gamma P\mu\,\Delta Q
\)
gives
\(
\|\Delta Q\|_{2,d_\mu,\mu}
\le
\|\Delta u\|_{2,d_\mu,\mu}
+
\gamma\|P\mu\,\Delta Q\|_{2,d_\mu,\mu}.
\)
By Jensen's inequality and stationarity of $d_\mu$,
\[
\|P\mu f\|_{2,d_\mu,\mu}^2
\le \int P\mu(f^2)\,d(d_\mu\otimes\mu)
=\int \mu(f^2)\,d(d_\mu P_\mu)
=\|f\|_{2,d_\mu,\mu}^2.
\]
Therefore
\[
(1-\gamma)\|\Delta Q\|_{2,d_\mu,\mu}
\le
\|\Delta u\|_{2,d_\mu,\mu}.
\]

Finally,
\begin{align*}
\|\Delta u\|_{2,d_\mu,\mu}^2
&=
\int (\Delta u(s,a))^2\,d_\mu(s)\mu(da\mid s) \\
&=
\int (\Delta u(s,a))^2
\frac{dd_\mu}{d\rho}(s)\frac{\mu(a\mid s)}{\pi(a\mid s)}
\,d\nu_\pi(s,a) \\
&\le
C_{d_\mu/\rho}C_{\mu/\pi}\|\Delta u\|_{2,\mathrm{beh}}^2.
\end{align*}
Thus
\begin{equation}
\label{eq:DeltaQ-lambda}
\|\Delta Q\|_{2,d_\mu,\mu}
\le
\frac{\sqrt{C_{d_\mu/\rho}C_{\mu/\pi}}}{1-\gamma}
\|u^\star-\hat u\|_{2,\mathrm{beh}}.
\end{equation}

Combining \eqref{eq:Q-triangle}, \eqref{eq:DeltaQ-beh}, and \eqref{eq:DeltaQ-lambda}, we obtain
\[
\|\widetilde\Delta_Q\|_{2,\mathrm{beh}}
\le
\|Q^\mu_{\hat u-g}-\hat Q\|_{2,\mathrm{beh}}
+
\left(
1+\frac{\gamma\sqrt{C_{\nu_\pi P/d_\mu}\,C_{d_\mu/\rho}C_{\mu/\pi}}}{1-\gamma}
\right)
\|u^\star-\hat u\|_{2,\mathrm{beh}}.
\]
Substituting this into \eqref{eq:reward-to-Q} proves the claim.
\end{proof}

\subsection{Technical Lemmas}
\label{appendix:lemmas}

\begin{lemma}
\label{lemma::KLell2}
Under Assumption~\ref{assump:policy-gen}, set
\(C_p:=\sqrt{2/\underline p}\). Then, with probability at least
\(1-\delta\),
\[
\|\hat u-u^\star\|_{2, \mathrm{beh}}
\le
C_p\,\rho_{\pi}(n,\delta).
\]
\end{lemma}

\begin{proof}
Work on the policy event in Assumption~\ref{assump:policy-gen}, and put
$x(s,a)=\log\{\hat\pi(a\mid s)/\pi(a\mid s)\}$.
Both probabilities are at least $\underline p$, so
$x\in[\log\underline p,-\log\underline p]$.
For $h(x)=e^x-1-x$, we have $h(0)=h'(0)=0$ and
$h''(x)=e^x\ge\underline p$ on this interval. Taylor's formula therefore
gives $h(x)\ge\underline p x^2/2$.
Since $\E_{\pi(\cdot\mid s)}[e^x]=1$,
\[
\KL\{\pi(\cdot\mid s)\,\|\,\hat\pi(\cdot\mid s)\}
=\E_{\pi(\cdot\mid s)}[-x]
=\E_{\pi(\cdot\mid s)}[h(x)]
\ge\frac{\underline p}{2}\E_{\pi(\cdot\mid s)}[x^2].
\]
Integrating over $s\sim\rho$ and applying the assumed KL bound gives
$\|\hat u-u^\star\|_{2,\mathrm{beh}}
\le\sqrt{2/\underline p}\,\rho_\pi(n,\delta)$ on this event.
\end{proof}

\begin{lemma}[Per-iteration error bound]
\label{lemma::periter}
Under Assumptions~\ref{assump:policy-gen}--\ref{assump:q-gen}, fix
\(\delta_\pi,\delta_Q\in(0,0.5)\). On the intersection of the policy event at
level \(\delta_\pi\) and the \(k\)th regression event at level \(\delta_Q\),
\[
\|\mathcal T_{\hat u}^\mu(\hat Q^{(k-1)})-\hat Q^{(k)}\|_{2,\mathrm{beh}}
\le
\sqrt{2/\underline p}\,\rho_\pi(n,\delta_\pi)
+\rho_{Q}(n,\delta_Q)+\varepsilon_{\mathcal F}.
\]
This event has probability at least \(1-\delta_\pi-\delta_Q\).
\end{lemma}

\begin{proof}
Write $\|\cdot\|=\|\cdot\|_{2,\mathrm{beh}}$, and let
$R_k^\star=\inf_{f\in\mathcal F}
\|\mathcal T_{\hat u}^\mu\hat Q^{(k-1)}-f\|$.
On the $k$th regression event, the definition of excess risk gives
\[
\|\mathcal T_{\hat u}^\mu\hat Q^{(k-1)}-\hat Q^{(k)}\|
\le\{(R_k^\star)^2+\rho_Q(n,\delta_Q)^2\}^{1/2}
\le R_k^\star+\rho_Q(n,\delta_Q).
\]
The targets under $\hat u$ and $u^\star$ differ by $\hat u-u^\star$.
Since $\hat Q^{(k-1)}\in\mathcal F$ (including the initialization
$\hat Q^{(0)}=0$), the triangle inequality and
$\|h\|_{2,\mathrm{beh}}\le\|h\|_\infty$ imply
\begin{align*}
R_k^\star
&\le\|\hat u-u^\star\|
+\inf_{f\in\mathcal F}\|\mathcal T_{u^\star}^\mu\hat Q^{(k-1)}-f\|\\
&\le\|\hat u-u^\star\|+\varepsilon_{\mathcal F}.
\end{align*}
On the policy event, Lemma~\ref{lemma::KLell2} bounds the first term by
$\sqrt{2/\underline p}\,\rho_\pi(n,\delta_\pi)$.
Combining the inequalities proves the claim on the intersection of the two
events, whose probability is at least $1-\delta_\pi-\delta_Q$ by a union bound.
\end{proof}

\begin{assumption}[Concentrability of propagated state-action distributions]
\label{assump:fqe-conc}
Assume that
\[
C_{\mathrm{conc}}
:=
\sup_{m\ge 0}
\left\|
\frac{d\bigl(\nu_\pi (P\mu)^m\bigr)}{d\nu_\pi}
\right\|_\infty
< \infty.
\]
\end{assumption}

\begin{lemma}[Coverage implies propagated concentrability]
\label{lem:coverage-implies-conc}
Let \(\nu_\pi:=\rho\otimes \pi\) and \(\nu_\mu:=d_\mu\otimes \mu\), where
\(d_\mu\) is a stationary distribution of \(P_\mu\). Suppose that
\[
C_{\mu/\pi}
:=
\left\|
\frac{\mu(a\mid s)}{\pi(a\mid s)}
\right\|_\infty
<\infty,
\qquad
C_{d_\mu/\rho}
:=
\left\|
\frac{d d_\mu(s)}{d\rho(s)}
\right\|_\infty
<\infty,
\]
and
\[
C_{\nu_\pi P/d_\mu}
:=
\left\|
\frac{d(\nu_\pi P)(s)}{d d_\mu(s)}
\right\|_\infty
<\infty.
\]
Then
\[
\sup_{m\ge 0}
\left\|
\frac{d\bigl(\nu_\pi (P\mu)^m\bigr)}{d\nu_\pi}
\right\|_\infty
\le
\max\Bigl\{
1,\,
C_{\nu_\pi P/d_\mu}\,C_{d_\mu/\rho}\,C_{\mu/\pi}
\Bigr\}.
\]
In particular, Assumption~\ref{cond::coverage} implies
Assumption~\ref{assump:fqe-conc}.
\end{lemma}

\begin{proof}
The density ratio is one when $m=0$. For $m\ge1$, let
$\eta_m=(\nu_\pi P)P_\mu^{m-1}$, so that
$\nu_\pi(P\mu)^m=\eta_m\otimes\mu$.
The one-step coverage assumption gives
$\nu_\pi P\le C_{\nu_\pi P/d_\mu}d_\mu$ as measures.
Applying $P_\mu^{m-1}$ preserves this inequality, and stationarity gives
\[
\eta_m\le C_{\nu_\pi P/d_\mu}d_\mu P_\mu^{m-1}
=C_{\nu_\pi P/d_\mu}d_\mu.
\]
Consequently,
\[
\frac{d\{\nu_\pi(P\mu)^m\}}{d\nu_\pi}(s,a)
=\frac{d\eta_m}{dd_\mu}(s)
\frac{dd_\mu}{d\rho}(s)\frac{\mu(a\mid s)}{\pi(a\mid s)}
\le C_{\nu_\pi P/d_\mu}C_{d_\mu/\rho}C_{\mu/\pi}.
\]
Together with the case $m=0$, this proves the stated bound.
Each factor is at least one, since it bounds a density ratio between
probability measures. Thus the maximum in the statement also equals
$C_{\mathrm{cov}}$, as used in the main text.
\end{proof}

\subsection{GenPQR + FQE Bound}

Assumption~\ref{assump:q-gen} already makes $\mathcal F$ uniformly bounded,
and Algorithm~\ref{alg:simple-irl} initializes at $0\in\mathcal F$.
On the policy event, $\|\hat u\|_\infty\le-\log\underline p$; since $g$ is
bounded, $Q^\mu_{\hat u-g}$ and all fitted iterates are bounded as well.

\begin{lemma}[FQE error bound]
\label{lem:fqe-bound}
Let \(\mathcal T^\mu_u f := u-g+\gamma P\mu f\), and let
\(Q^\mu_{\hat u-g}\) denote the unique fixed point of
\(\mathcal T^\mu_{\hat u}\). Suppose \(\hat Q^{(k)}\) is produced by
Algorithm~\ref{alg:simple-irl}. Under
Assumptions~\ref{assump:policy-gen}--\ref{assump:q-gen},
Assumption~\ref{assump:fqe-conc}, with probability at least
\(1-\delta\),
\begin{align*}
\|Q^\mu_{\hat u-g}-\hat Q^{(K)}\|_{2,\mathrm{beh}}
&\le
\sqrt{C_{\mathrm{conc}}}\,\gamma^K
\|Q^\mu_{\hat u-g}-\hat Q^{(0)}\|_{\infty}\\
&\qquad+
\kappa_{\mathrm{conc}}
\Bigl(
\varepsilon_{\mathcal F}
+\rho_Q(n,\delta/(2K))
+ \sqrt{2/\underline p}\,\rho_\pi(n,\delta/2)
\Bigr).
\end{align*}
\end{lemma}

\begin{proof}
Write $Q_{\hat u}=Q^\mu_{\hat u-g}$ and let
$\xi_j=\mathcal T^\mu_{\hat u}\hat Q^{(j-1)}-\hat Q^{(j)}$.
Subtracting the fitted update from the fixed-point equation gives
\[
Q_{\hat u}-\hat Q^{(j)}
=\gamma P\mu(Q_{\hat u}-\hat Q^{(j-1)})+\xi_j.
\]
Iterating this identity yields
\begin{equation}
\label{eq:fqe-error-recursion}
Q_{\hat u}-\hat Q^{(K)}
=\gamma^K(P\mu)^K(Q_{\hat u}-\hat Q^{(0)})
+\sum_{j=1}^K\gamma^{K-j}(P\mu)^{K-j}\xi_j.
\end{equation}
For $m\ge1$, Jensen's inequality and Assumption~\ref{assump:fqe-conc} give
\begin{align*}
\|(P\mu)^m f\|_{2,\mathrm{beh}}^2
&\le\int(P\mu)^m(f^2)\,d\nu_\pi
=\int f^2\,d\{\nu_\pi(P\mu)^m\}\\
&\le C_{\mathrm{conc}}\|f\|_{2,\mathrm{beh}}^2.
\end{align*}
For $m=0$, the norm is unchanged. The initialization term also satisfies
$\|(P\mu)^K(Q_{\hat u}-\hat Q^{(0)})\|_{2,\mathrm{beh}}
\le\|Q_{\hat u}-\hat Q^{(0)}\|_\infty$.

Let $\mathcal E$ be the intersection of the policy event at level $\delta/2$
and the $K$ regression events at level $\delta/(2K)$.
A union bound gives $\Prob(\mathcal E)\ge1-\delta$; independence of these
events is not needed. On $\mathcal E$, Lemma~\ref{lemma::periter} gives,
simultaneously for all $j\in[K]$,
\[
\|\xi_j\|_{2,\mathrm{beh}}\le e_K
:=\varepsilon_{\mathcal F}+\rho_Q(n,\delta/(2K))
+\sqrt{2/\underline p}\,\rho_\pi(n,\delta/2).
\]
Applying these bounds to \eqref{eq:fqe-error-recursion}, and retaining
coefficient one for the unpropagated residual $\xi_K$, gives
\begin{align*}
\|Q_{\hat u}-\hat Q^{(K)}\|_{2,\mathrm{beh}}
&\le\gamma^K\|Q_{\hat u}-\hat Q^{(0)}\|_\infty
+\left(1+\sqrt{C_{\mathrm{conc}}}
\sum_{m=1}^{K-1}\gamma^m\right)e_K\\
&\le\gamma^K\|Q_{\hat u}-\hat Q^{(0)}\|_\infty
+\kappa_{\mathrm{conc}}e_K.
\end{align*}
This implies the stated bound because $C_{\mathrm{conc}}\ge1$.
\end{proof}
 
\begin{proof}[Proof of Theorem~\ref{thm:pqr-fqe}]
Let $\mathcal E$ be the intersection of the policy event at level $\delta/2$
and the $K$ regression events at level $\delta/(2K)$, as in the proof of
Lemma~\ref{lem:fqe-bound}. Then $\Prob(\mathcal E)\ge1-\delta$.
On this same event, Lemmas~\ref{lemma::KLell2} and~\ref{lem:fqe-bound} give
\begin{align*}
\|u^\star-\hat u\|_{2,\mathrm{beh}}
&\le\sqrt{2/\underline p}\,\rho_\pi(n,\delta/2),\\
\|Q^\mu_{\hat u-g}-\hat Q^{(K)}\|_{2,\mathrm{beh}}
&\le\sqrt{C_{\mathrm{conc}}}\,\gamma^K
\|Q^\mu_{\hat u-g}-\hat Q^{(0)}\|_\infty\\
&\quad+\kappa_{\mathrm{conc}}\left\{
\varepsilon_{\mathcal F}+\rho_Q(n,\delta/(2K))
+\sqrt{2/\underline p}\,\rho_\pi(n,\delta/2)\right\}.
\end{align*}
Substitute these inequalities into Theorem~\ref{thm:reward-recovery}.
The two contributions involving $\rho_\pi$ have combined coefficient
\[
\sqrt{2/\underline p}\left(
\kappa_{\mathrm{conc}}+1+
\frac{\gamma\sqrt{C_{\mathrm{cov}}}}{1-\gamma}\right)
\]
inside the factor $\sqrt{C_{\mu/\pi}}$. This is the claimed bound.
\end{proof}

\section{Continuous-Action Extension}
\label{app:continuous-actions}

This appendix gives the continuous-action analogue of the identification
result and algorithm. Finite-action sums are replaced by integrals, and the
softmax policy is replaced by a Boltzmann density.

Let \((\mathcal A,\mathcal B_{\mathcal A})\) be an action space equipped with a
finite reference measure \(\lambda\). For a bounded measurable state--action
function \(f\), define the log-partition operator
\[
\LSE f(s)
:=
\log\int_{\mathcal A}\exp\{f(s,a)\}\,\lambda(da),
\]
and the associated Boltzmann policy
\[
\pi_f(da\mid s)
=
\exp\{f(s,a)-\LSE f(s)\}\,\lambda(da).
\]
For any policy or reference kernel \(\eta(\cdot\mid s)\), write
\[
(\eta f)(s)
:=
\int_{\mathcal A} f(s,a)\,\eta(da\mid s),
\qquad
(P\eta f)(s,a)
:=
\E\!\left[
\int_{\mathcal A} f(s',a')\,\eta(da'\mid s')
\;\middle|\; s,a
\right].
\]
The normalization is the same statewise affine constraint as in the main text,
\[
(\mu r)(s)=g(s),
\]
now with integrals in place of sums.

\begin{assumption}[Continuous-action regularity]
\label{assump:continuous-actions}
The discount satisfies \(\gamma\in[0,1)\), the anchor \(g\) is bounded, and \(P\)
maps bounded measurable state functions to bounded measurable state--action
functions. The behavior policy admits a strictly positive density
\(p_\pi(a\mid s)\) with respect to \(\lambda\), and
\[
u^\star(s,a):=\log p_\pi(a\mid s)
\]
is bounded with \(\LSE u^\star(s)=0\) for all \(s\). The normalization kernel
\(\mu(\cdot\mid s)\) is a Markov kernel on \(\mathcal A\), dominated by
\(\lambda\), and is used only to define the normalization and the evaluation
operator \(P\mu\).
\end{assumption}

\begin{theorem}[Continuous-action fixed-point identification]
\label{thm:continuous-action-identification}
Under Assumption~\ref{assump:continuous-actions}, let \(Q^\star\) be the unique
bounded solution of
\[
Q^\star(s,a)
=
u^\star(s,a)-g(s)+\gamma(P\mu Q^\star)(s,a).
\]
Then, modulo \(\lambda\)-null action sets at each state, the unique bounded
normalized reward representative inducing the Boltzmann policy
\(\pi(da\mid s)=\exp\{u^\star(s,a)\}\lambda(da)\) is
\[
r^\star(s,a)
=
Q^\star(s,a)-(\mu Q^\star)(s)+g(s),
\qquad
v^\star(s,a)
=
-(P\mu Q^\star)(s,a).
\]
Equivalently, \((r^\star,v^\star)\) satisfies the continuous soft Bellman
equation \(v^\star=P\LSE(r^\star+\gamma v^\star)\), the normalization
\(\mu r^\star=g\), and induces the observed Boltzmann density.
\end{theorem}

\begin{proof}
Since $P\mu$ has sup-norm operator norm at most one, the map
$Q\mapsto u^\star-g+\gamma P\mu Q$ is a $\gamma$-contraction on bounded
state--action functions. Its fixed point $Q^\star$ therefore exists and is
unique. Define $(r^\star,v^\star)$ as in the theorem. Then $\mu r^\star=g$,
and the fixed-point equation gives
\[
r^\star+\gamma v^\star
=Q^\star-\mu Q^\star+g-\gamma P\mu Q^\star
=u^\star-\mu Q^\star.
\]
Because $\LSE u^\star=0$, we obtain
$\LSE(r^\star+\gamma v^\star)=-\mu Q^\star$.
Thus $P\LSE(r^\star+\gamma v^\star)=v^\star$, and the induced log-density is
$u^\star$. This proves existence of a normalized pair inducing $\pi$.

For uniqueness, let $(r,v)$ be another bounded normalized pair satisfying
the soft Bellman equation and inducing $\pi$.
Equality of the induced densities gives
$r+\gamma v=u^\star+c$ for the bounded state function
$c=\LSE(r+\gamma v)$, modulo $\lambda$-null action sets at each state.
Bellman feasibility and normalization then imply
\[
v=Pc,\qquad r=u^\star+c-\gamma Pc,\qquad
(I-\gamma P_\mu)c=g-\mu u^\star.
\]
Here $P_\mu c=\mu Pc$; domination of $\mu$ by $\lambda$ justifies integrating
the actionwise identity. Since $P_\mu$ has sup-norm operator norm at most one,
\[
c=\sum_{j=0}^\infty\gamma^jP_\mu^j(g-\mu u^\star)
\]
is the unique bounded solution. Hence $(r,v)$ is unique modulo the stated
null sets.
\end{proof}

\begin{algorithm}[H]
\caption{\textsc{Continuous-Action GenPQR}}
\label{alg:continuous-genpqr}
\begin{algorithmic}[1]
\INPUT Transitions \(\{(s_i,a_i,s_i')\}_{i=1}^n\), normalization kernel
\(\mu(\cdot\mid s)\), anchor \(g(s)\), discount \(\gamma\), class
\(\mathcal F\), iterations \(K\)
\STATE Estimate a conditional log-density
\(\hat u(s,a)\approx \log p_\pi(a\mid s)\) with respect to the reference
measure \(\lambda\)
\STATE Initialize \(\hat Q^{(0)}(s,a)\gets 0\)
\FOR{$k=1,\ldots,K$}
    \STATE For each \(i\), compute
    \[
    m_i^{(k-1)}
    \gets
    \int_{\mathcal A}
    \hat Q^{(k-1)}(s_i',a')\,\mu(da'\mid s_i')
    \]
    by quadrature or Monte Carlo samples from \(\mu(\cdot\mid s_i')\)
    \STATE Set
    \[
    y_i
    \gets
    \hat u(s_i,a_i)-g(s_i)+\gamma m_i^{(k-1)}
    \]
    \STATE Fit \(\hat Q^{(k)}\in\mathcal F\) by regressing \(y_i\) on
    \((s_i,a_i)\)
\ENDFOR
\STATE Set \(\hat Q\gets \hat Q^{(K)}\)
\OUTPUT \textbf{Reward:}
\[
\hat r(s,a)
=
\hat Q(s,a)
-
\int_{\mathcal A}\hat Q(s,a')\,\mu(da'\mid s)
+
g(s),
\]
with the final integral evaluated by the same quadrature or Monte Carlo rule.
\end{algorithmic}
\end{algorithm}

\end{document}